\documentclass{article}

\usepackage[main, final]{neurips_2026}

\usepackage[utf8]{inputenc} 
\usepackage[T1]{fontenc}    
\usepackage{hyperref}       
\usepackage{url}            
\usepackage{booktabs}       
\usepackage{tabularx}
\usepackage{rotating}
\usepackage{multirow}
\usepackage{amsfonts}       
\usepackage{nicefrac}       
\usepackage{microtype}      
\usepackage{xcolor}         

\usepackage{algorithm}
\usepackage{algpseudocode}
\usepackage{dsfont}
\usepackage{enumitem}

\usepackage[mode=buildnew]{standalone}
\usepackage{tikz}
\usetikzlibrary{external}
\usepackage{pgfplots}
\usepackage{adjustbox}
\pgfplotsset{compat=1.18}
\pgfplotsset{
  every axis/.append style={
    font=\scriptsize,
    label style={yshift=-0.8ex},
    tick label style={yshift=-0.4ex},
    title style={yshift=-1.6ex},
  },
  xlabel style={yshift=1.2ex},
  ylabel style={xshift=-0.8ex},
  ticklabel style={
    inner sep=0.5pt,
  }
}

\usepackage{amsmath}
\usepackage{amssymb}
\usepackage{mathtools}
\usepackage{amsthm}

\theoremstyle{plain}
\newtheorem{theorem}{Theorem}[section]
\newtheorem{proposition}[theorem]{Proposition}
\newtheorem{lemma}[theorem]{Lemma}
\newtheorem{corollary}[theorem]{Corollary}
\theoremstyle{definition}

\newtheorem{assumption}[theorem]{Assumption}
\theoremstyle{remark}
\newtheorem{remark}[theorem]{Remark}

\DeclareMathOperator*{\argmax}{argmax}

\DeclareMathOperator*{\calA}{\mathcal{A}}
\DeclareMathOperator*{\calB}{\mathcal{B}}
\DeclareMathOperator*{\calS}{\mathcal{S}}
\DeclareMathOperator*{\calZ}{\mathcal{Z}}
\DeclareMathOperator*{\argmin}{arg\,min}

\usepackage[textsize=tiny]{todonotes}

\title{Breaking the Grid: Distance-Guided Reinforcement \\ Learning in Large Discrete Action Spaces}

\author{
  Heiko Hoppe\textsuperscript{1} \quad
  Fabian Akkerman\textsuperscript{2} \quad
  Wouter van Heeswijk\textsuperscript{2} \quad
  Maximilian Schiffer\textsuperscript{3} \\
  \textsuperscript{1}Technical University of Munich \quad \textsuperscript{2}University of Twente \quad \textsuperscript{3}HEC Paris \\
  \texttt{heiko.hoppe@tum.de} \\ \texttt{\{f.r.akkerman,w.j.a.vanheeswijk\}@utwente.nl} \\ \texttt{schiffer@hec.fr}
}

\begin{document}

\maketitle

\begin{abstract}
  Reinforcement Learning (RL) is increasingly applied to large-scale decision-making problems like logistics, scheduling, and recommender systems, but existing algorithms struggle with the curse of dimensionality in such large discrete action spaces. We propose \emph{Distance-Guided Reinforcement Learning} (DGRL), combining Sampled Dynamic Neighborhoods and Distance-Based Updates to enable efficient RL in problems with up to $10^{36}$ actions. Unlike prior methods, DGRL performs stochastic volumetric exploration and transforms policy optimization into a stable regression task, decoupling gradient variance from action space cardinality. On structured tasks, DGRL provably guarantees local value improvement. DGRL naturally handles action spaces without regular semantic structure and hybrid continuous-discrete action spaces. We demonstrate performance improvements of up to 66\% against state-of-the-art benchmarks across regularly and irregularly structured environments, while simultaneously improving convergence speed and computational complexity.\footnote{Our code is available at \url{https://github.com/HeikoFHH/Distance-Guided-RL}.}
\end{abstract}

\section{Introduction}
\label{sec:introduction}

Reinforcement Learning (RL) has achieved remarkable success in continuous and small-scale discrete domains. However, many real-world applications -- ranging from logistics planning and scheduling to recommender systems and robot control -- require agents to navigate action spaces that exhibit both large cardinality and complex structure. These spaces often include high-dimensional discrete components, categorical variables, or combinations of discrete and continuous parameters, presenting a fundamental ``curse of cardinality'' that renders traditional approaches ineffective.

Scalable methods map a continuous proto-action to a discrete one, but rounding is suboptimal \citep{hasselt2009}, nearest-neighbor lookups rely on expensive action enumeration \citep{dulac2015}, and grid-based searches are brittle in irregular spaces \citep{akkerman2024dynamic}.
No existing method scales to hybrid spaces.
We close this gap with \emph{Distance-Guided Reinforcement Learning} (DGRL), an actor-critic algorithm with two components: \emph{Sampled Dynamic Neighborhoods} (SDN) sample discrete candidates from a Chebyshev ($L_\infty$) neighborhood of a continuous proto-action and select among them using the critic, while \emph{Distance-Based Updates} (DBU) regress the proto-action onto a high-value target built from sampled candidates.

We prove that SDN reduces search complexity from exponentially to linear as action space dimensionality increases and that DBU's gradient variance is independent of action space cardinality $|\calA|$. In structured tasks, DGRL ensures local value improvement. In discrete and hybrid environments, DGRL scales up to $10^{36}$, outperforming state-of-the-art benchmarks by up to 66\%.

\subsection{Problem Setup and Notation}\label{sec:problem_setup}

We consider Markov Decision Processes (MDPs) defined by the tuple $(\calS, \calA, P, R, \gamma)$, comprising a transition kernel $P(s'|s,a)$, a reward function $R(s,a)$, depending on states $s \in \calS$ and actions $a \in \calA$, and a discount factor $\gamma$. We focus on maximizing the expected return $J(\pi) = \mathbb{E}[\sum_{t=0}^T \gamma^t r_t]$.

We specifically address two challenging regimes: \emph{Large Discrete Action Spaces} and \emph{Large Hybrid Action Spaces}. To efficiently navigate such massive spaces, algorithms can benefit from exploiting their dimensionality and metric structure. We therefore formalize the dimensionality constraints, metric structures, and continuous relaxations that underpin our approach.

\textbf{Dimensionality and Structure.}
We classify discrete action spaces based on two properties dictating scalability. First, we distinguish \textit{univariate} actions (a single index) and \textit{multivariate} actions. An action space $\calA$ is multivariate if each action $a$ is a vector of $N$ components $a = (a_1, \dots, a_N)$, where each component $a_n$ is chosen from a sub-space $\calA_n$. In these settings, the cardinality $|\calA|$ grows exponentially with the dimensionality $N$, rendering enumeration intractable. Second, we distinguish action spaces according to their underlying metric. We distinguish three component types: \\
\textit{Numerical:} Actions represent quantitative values (e.g., prices). Distance is natural (Euclidean). \\
\textit{Ordinal:} Actions represent ordered options (e.g., intensity). Distance is defined by rank. \\
\textit{Categorical:} Actions are unordered (e.g., items in a catalog). Distance has no direct interpretation.

We call an action space \emph{structured} if its integer coordinates are ordered and equidistant with respect to the semantics of the actions, such that index adjacency implies semantic proximity (numerical and ordinal components), and \emph{irregularly structured} if this alignment is absent, because the components are categorical or semantic spacing between adjacent indices is uneven. Unlike algorithms relying on regular grids, e.g., \citet{akkerman2024dynamic}, DGRL can handle irregularly structured action spaces.

\textbf{Relaxed Action Spaces.}
To navigate these large spaces, we operate in a continuous relaxation. Let $\calA \subset \mathbb{Z}^N$ be the set of valid discrete actions. We define the \textit{Relaxed Action Space} $\calA' \subseteq \mathbb{R}^N$ such that $\calA \subset \calA'$. The actor network $\varphi_\theta: \calS \to \calA'$ outputs a continuous \textit{proto-action} $\hat{a} \in \calA'$. In numerical and ordinal action spaces, the structure of $\calA$ directly translates to $\calA'$, allowing algorithms to exploit this structure in the relaxed action space. Since actions in $\calA$ are denoted by integer indices, $\calA'$ is the continuous relaxation of this action space.

\textbf{Hybrid Action Spaces.}
We investigate hybrid spaces $\calA \subset \mathbb{Z}^N \times \mathbb{R}^M$, where an action $a = (a_d, a_c)$ contains both discrete and continuous components. We are particularly interested in \textit{Parameterized Action Spaces} \citep{Hausknecht2016}, where the validity or semantic meaning of the continuous parameter $a_c$ depends on the discrete choice $a_d$. This creates a tight coupling between components, often rendering methods that apply independent optimization suboptimal.

\subsection{Related work}
\label{sec:literature}

Standard methods fail to scale to large-scale regimes due to three primary bottlenecks: gradient variance, computational tractability, and structural rigidity. Early approaches attempt to reduce dimensionality through factorization \citep{SallansAndGeoffrey2004, pazis2011} or symbolic representations \citep{CuiAndKhardon2016}, but these require substantial manual tuning. Representation learning methods \citep{chandak2019, whitney2019} often suffer from high sample complexity as they require learning dedicated representations for every discrete action. A more scalable paradigm leverages continuous actor-critic architectures by mapping continuous proto-actions to discrete counterparts. Simple rounding \citep{hasselt2009, vanvuchelen2022} is efficient but suboptimal in complex landscapes. To improve precision, the Wolpertinger architecture \citep{dulac2015} employs $k$-nearest neighbor searches, yet this becomes computationally intractable in high-dimensional spaces. Recently, \citet{akkerman2024dynamic} introduced Dynamic Neighborhood Construction (DNC) to navigate neighborhoods via simulated annealing; while scalable, it relies heavily on regular structures (e.g., numerical adjacency), making it brittle in irregularly structured domains. Finally, hybrid continuous-discrete action spaces pose unique challenges, as most existing algorithms treat components independently, failing to account for the tight coupling required in complex tasks \citep{Hausknecht2016, Fan2019, Peng2019, Delalleau2020}, or rely on hierarchical sequencing, incurring dependencies adversely affecting performance \citep{Xiong2018, Bester2019, Ma2021, Wei2018}. So far, no work has addressed the challenges incurred by large-scale hybrid action spaces.

\textbf{Contributions.} We propose \textit{Distance-Guided Reinforcement Learning (DGRL)}, building upon Wolpertinger of \citet{dulac2015} and DNC of \citet{akkerman2024dynamic}. DGRL introduces a volumetric projection loop composed of two synergistic pillars:
\begin{itemize}[leftmargin=*,noitemsep,topsep=0pt]
    \item \textbf{Sampled Dynamic Neighborhoods (SDN):} An action refinement method that performs stochastic volumetric exploration in the discrete neighborhood of a continuous proto-action. By utilizing the $L_\infty$ (Chebyshev) metric, SDN maintains a stable search volume as dimensionality $N \to \infty$ and reduces search complexity from exponential to linear.
    \item \textbf{Distance-Based Updates (DBU):} A denoised projection mechanism that transforms policy optimization into a stable regression task. We prove that DBU's gradient variance is independent of action cardinality $|\mathcal{A}|$, addressing the key bottleneck of gradient explosion in large spaces.
\end{itemize}
We establish that DGRL ensures local value improvement in structured tasks and enables joint optimization in hybrid action spaces, bypassing the commitment bottlenecks of hierarchical methods. Empirical results across mazes, logistics, and recommender systems show performance gains of up to 66\% against state-of-the-art benchmarks in spaces with $10^{36}$ actions, while improving stability and reducing computational overhead significantly.

\section{Methodology}
\label{sec:methodology}

We propose \emph{Distance-Guided Reinforcement Learning} (DGRL), a new algorithm for RL in action spaces with up to $10^{36}$ candidates. An off-policy actor-critic algorithm, our approach introduces a volumetric projection loop composed of two synergistic components: \textit{Sampled Dynamic Neighborhoods} (SDN) and \textit{Distance-Based Updates} (DBU). The actor generates a continuous proto-action that SDN uses as the starting point of a Chebyshev-constrained critic-based stochastic search to identify the discrete execution action. DBU updates the actor by minimizing the distance between the proto-action and a high-value target action constructed from perturbed action samples.

\subsection{Sampled Dynamic Neighborhood (SDN)}
\label{sec:sdn}

Given a proto-action $\hat a$, SDN constructs a Chebyshev-constrained search region and applies distance-based stochastic sampling within that region. Utilizing the critic for action selection, SDN enables efficient inference in high-dimensional action spaces. We visualize SDN in Fig.~\ref{fig:sdn} and Algorithm~\ref{alg:sdn}.

\textbf{Motivation.}
In large action spaces, identifying the optimal action $a$ effectively requires solving two sub-problems: finding the promising region of the state-action space, and identifying the specific optimal discrete action within that region. Previous methods have three shortcomings: i) Simple rounding is empirically incapable of translating proto-actions to optimal discrete actions \citep{hasselt2009}. ii) Deterministic lookup methods rely on pre-computed distance tables and lack stochastic exploration \citep{dulac2015}. iii) Grid-based stochastic search methods only explore along rigid grid lines of the action space, lacking volumetric support \citep{akkerman2024dynamic}.

These shortcomings motivate us to interpret $\hat a$ as a \emph{principled prior} for a volumetric stochastic action search: it contains information about the actor's action selection preferences, which we use to search its discrete neighborhood for $a$.

\begin{figure}
    \centering
    \vspace{-8mm}
    \includegraphics[width=0.85\linewidth]{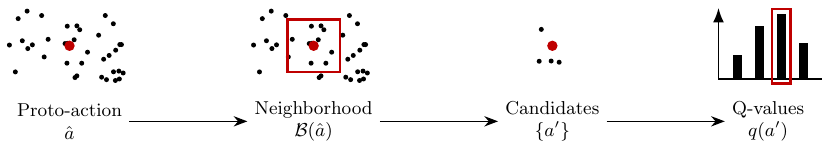}
    \vspace{-4mm}
    \caption{Schematic representation of SDN.}
    \label{fig:sdn}
    \vspace{-4mm}
\end{figure}

\textbf{Mechanism.}
Adopting the continuous-discrete mapping paradigm of \citet{hasselt2009}, our actor $\varphi_\theta(s)$ operates in the relaxed space $\calA'$, outputting a continuous \textit{proto-action} $\hat{a}$ that serves as the center of mass for a high-value region. As introduced by \citet{dulac2015}, the critic then acts as a local selector to refine this coarse estimate. In contrast to previous approaches, most notably Wolpertinger's $k$-nearest-neighbor lookup over a pre-computed table of all actions, we construct a dynamic neighborhood $\calB(\hat{a})$ around $\hat{a}$ bounded by a Chebyshev radius $L$, chosen for its dimensional invariance. When the action space dimensionality $N$ increases, standard Euclidean ($L_2$) radii must grow with $\sqrt{N}$ to encompass a fixed volume of independent coordinate variations, leading to vanishing sampling density. In contrast, Chebyshev ($L_\infty$) radii remain invariant of $N$, ensuring that sampling probabilities remain non-vanishing even as $N \rightarrow \infty$, formalized in App.~\ref{app:chebyshev_invariance}.

\textbf{Independent Coordinate Sampling.}
To avoid the combinatorial explosion associated with enumerating all integer points in this hypercube, we approximate the neighborhood via probabilistic sampling. We exploit the structure of multivariate spaces by sampling each dimension $n \in \{1, \dots, N\}$ independently. We interpret the proto-action $\hat{a}$ as the mode of a sampling distribution over $\calB(\hat{a})$ and draw $K$ candidate actions.
The probability of sampling a candidate $a'$ decreases with its distance from $\hat{a}$, effectively creating a soft trust region. To estimate probabilities, we can use a softmax over the negative absolute distances or apply a linear scheme:
\begin{equation}\label{eq:candidate_sampling}
    p(a'_n)=\frac{L-|\hat a_n - a'_n|+\tau_s}{\sum_{a'_{\calB,n}}L-|\hat a_n - a'_{\calB,n}|+\tau_s},
\end{equation}
for each dimension $n \in \{1, \dots, N\}$ independently, with $a'_{\calB}$ being integer options in $\calB(\hat a)$ and $\tau_s$ being the sampling temperature parameter. In our experiments, we use the linear scheme with $\tau_s=1$, which depends less on the magnitude of distances than the softmax scheme. By design, independent coordinate sampling fulfills the Chebyshev Distance constraint.
While the total action space $|\mathcal{A}|$ is massive, SDN only requires $K$ samples to obtain a local estimate of the value landscape. Because the $L_\infty$ trust region remains volumetrically stable in high dimensions (App.~\ref{app:chebyshev_invariance}), $K$ does not need to scale with $|\mathcal{A}|$. Instead, $K$ is governed by the smoothness of the latent manifold, allowing for high-performance inference with as few as 40 samples in spaces of size $10^{36}$. We provide a practical guide to choosing $L$ and $K$ in App.~\ref{app:hyperparams}. In general, we choose $L$ between 1 and 10, and $K$ between 10 and 100, maintaining computational efficiency even in large-scale action spaces.

Independent coordinate sampling reduces the search complexity from exponential $\mathcal{O}((2L)^{N})$ to linear $\mathcal{O}(N \cdot K)$, enabling scalability to dimensions where $N > 100$ while retaining local neighborhood coverage. Crucially, as a consistent statistical estimator (App.~\ref{app:consistency}), SDN maintains linear complexity regardless of the total action cardinality $|A|$, providing a computationally tractable alternative to exhaustive or axial search in high-dimensional manifolds.

We utilize the critic $\psi_\omega$ for action selection: first, the critic evaluates all candidate actions, creating a list of Q-values $q_{(k)} = Q_{\psi_\omega}(s, a'_{(k)})$. During testing, we select $a = \argmax_k q_{(k)}$. To ensure exploration during training, we sample $a$ from $\{a'\}$ based on $q_{(k)}$, using either a softmax over the Q-values or a rank-based probability metric:
\begin{equation}\label{eq:action_sampling}
    p(a)=\frac{\tau_e^{\text{rank}(q(a))}}{\sum_{a'}\tau_e^{\text{rank}(q(a'))}},
\end{equation}
with $\text{rank}(q(a))$ being the descending rank of the action according to its Q-value and $\tau_e$ being the exploration temperature. In our experiments, we use the rank-based scheme with $\tau_e=0.8$, which is less influenced by the magnitude of Q-values than the softmax scheme.

To ensure these samples map correctly to the valid action space, we include a specific scaling step. This maps the actor's continuous output (typically clipped to a feasible range $[\hat{a}_\text{min}, \hat{a}_\text{max}]$) to the integer coordinates of the discrete action space $[A_\text{min}, A_\text{max}]$:
\[
    \hat{a}_\text{scaled} = \frac{\hat{a}_\text{clipped} - \hat{a}_\text{min}}{\hat{a}_\text{max} - \hat{a}_\text{min}} (A_\text{max} - A_\text{min}) + A_\text{min}.
\]

\begin{algorithm}[h]
  \caption{Sampled Dynamic Neighborhood (SDN)}
  \label{alg:sdn}
  \begin{algorithmic}
    \State {\bfseries Input:} state $s$, actor $\varphi_\theta$, critic $\psi_\omega$, radius $L$, samples $K$
    \State \textbf{1. Estimation:} Compute proto-action $\hat a \leftarrow \varphi_\theta(s)$
    \State \textbf{2. Scaling:} Map $\hat a$ to action space bounds via scaling formula
    \State \textbf{3. Independent Sampling:} Sample $K$ $a'_{(k)} \sim p(a')$ from $\calB_{L_\infty}(\hat{a}, L)$ \eqref{eq:candidate_sampling}
    \State \textbf{4. Evaluation:} Compute $Q$-values $q_{(k)} = Q_{\psi_\omega}(s, a'_{(k)})$
    \State \textbf{5. Selection:}
    \If{training}
        \State Sample $a \sim p(a)$ from $\{a'\}$ \eqref{eq:action_sampling}
    \Else
        \State Select $a = \argmax_k q_{(k)}$
    \EndIf
  \end{algorithmic}
\end{algorithm}

\textbf{Benefits of SDN.}
Compared to Wolpertinger and DNC, SDN offers four key advantages: i) SDN constructs neighborhoods during inference, eliminating the necessity of storing a distance lookup table like Wolpertinger. ii) SDN samples actions along independent dimensions, eliminating dependence on a metric action space like DNC. iii) SDN's dimension-independent sampling is easily parallelizable, empirically only scaling with $L$, in contrast to sequential candidate action generation of Wolpertinger and DNC. iv) SDN utilizes the spacial information of $\hat a$, in contrast to Wolpertinger and DNC, which fail to explore neighborhoods extensively and exhaustively independently of structure.

\textbf{Limitations.}
While finding the optimum of a neighborhood $\calB(\hat a)$ does not require $K$ to grow with $|\calA|$, it requires $K$ to grow with the dimensionality $N$: we set $K \approx L \cdot N$, such that the cost $\mathcal{O}(N \cdot K)$ is quadratic in $N$ (App.~\ref{app:consistency} and~\ref{app:hyperparams}). Since $K \approx L \cdot N$ samples cannot jointly cover a specific action among $(2L+1)^N$ candidates, DGRL relies on DBU to move the proto-action towards high-value actions, while SDN refines action selection locally (App.~\ref{app:volumetric_supportexploration}). Moreover, DGRL relies on a fixed representation of the action space. Its performance on irregularly structured action spaces, e.g., our Recommender environment, is an empirical result, not a theoretical guarantee.

\subsection{Distance-Based Updates (DBU)}
\label{sec:dbu}

With the search mechanism established, we require a compatible update rule. We propose the use of regression-based distance minimization losses towards softmax-constructed target actions.

\textbf{Motivation.}
When considering large discrete action spaces, standard RL losses are ill-suited: Policy Gradients applied to $\calA$ suffer from variance explosion as the probability mass $\pi(a|s)$ vanishes in large action spaces. When applied to the relaxed action space $\calA'$, Policy Gradients face a mismatch between the parametrized probability of choosing a proto-action $\hat a$ given $s$ and the probability of choosing an action $a$ given $s$ and the mapping procedure. This adds noise to the gradients and destabilizes learning. Deterministic policy gradients fail because the critic $\psi_\omega$ is trained only on valid discrete actions $a \in \calA$. The evaluation of $\nabla_a Q(s, \hat{a})$ for a continuous proto-action lying between valid actions is therefore inaccurate. Furthermore, the Q-function in large action spaces may be highly non-convex, necessitating the use of algorithms that can escape local optima \citep[cf.][]{Jain2025}.

\begin{figure}
    \centering
    \vspace{-8mm}
    \includegraphics[width=0.85\linewidth]{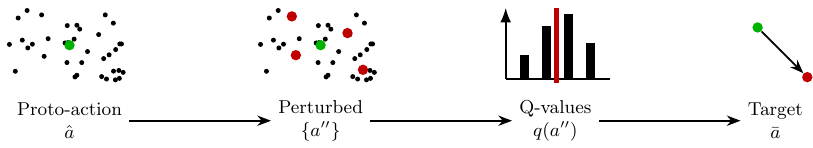}
    \vspace{-4mm}
    \caption{Schematic representation of DBU.}
    \label{fig:dbu}
    \vspace{-6mm}
\end{figure}

We propose DBU to circumvent these issues by transforming the policy update into a \textit{supervised regression} task against a stable, non-stationary target. Our loss function extends the paradigms proposed by \citet{abdolmaleki2018} and \citet{hoppe2025structured}, adapting them to the specific geometry of relaxed large discrete spaces. Intuitively, since action sampling probabilities are inversely related to their distance to $\hat a$ in SDN, distance to high-value target action is a natural loss. DBU creates such high-value targets and pushes the actor towards them, complementing the action selection of SDN. We visualize DBU in Fig.~\ref{fig:dbu} and outline it in Algorithm~\ref{alg:dbu}.

\textbf{Target Construction: The Softmax Average.}
Using a Gaussian, we first perturb and round the proto-action $\hat{a}$ to generate a set of $M$ auxiliary candidates $\{a''\}$ (distinct from the execution candidates in SDN to search beyond $\calB(\hat a)$). To generate a reliable training signal, we do not simply use the best candidate action, as this $\argmax_iq(s,a''_{(i)})$ is sensitive to critic noise -- given learned Q-values may not reflect the ground truth, the argmax can be the result of critic overestimation. Instead, we compute the softmax-weighted average of the candidates:
\begin{equation}\label{eq:target_action}
    \bar{a} = \sum_{i=1}^M a''_{(i)} \cdot \frac{\exp(Q(s, a''_{(i)})/\tau_b)}{\sum_{j=1}^M \exp(Q(s, a''_{(j)})/\tau_b)}.
\end{equation}
This operation acts as a denoising filter, producing a target $\bar{a}$ that represents the expected location of the optimal action. By averaging over candidates, it pushes the target away from narrow local optima and smoothens the loss landscape. In our experiments, we use $M=40$ and $\tau_b=0.01$.

\textbf{Regression Update.}
We update the actor weights $\theta$ to minimize the distance to this target:
\begin{equation}\label{eq:actor_update}
    J(\theta) = \|\varphi_\theta(s) - \bar{a}\|_2^2,
\end{equation}
in practice utilizing Huber losses. This provides a dense, low-variance gradient signal that points directly toward the high-value target action, independent of the cardinality of $\calA$. The update effectively pulls the proto-action $\hat{a}$ towards high-value discrete actions found by the critic.

\begin{algorithm}[h]
  \caption{Distance-Based Updates (DBU)}
  \label{alg:dbu}
  \begin{algorithmic}
    \State {\bfseries Input:} state $s$, actor $\varphi_\theta$, critic $\psi_\omega$, noise $\sigma_b$
    \State \textbf{1. Perturbation:} Generate $M$ candidates $a''_{(i)} \sim \mathrm{round}(\mathcal{N}(\varphi_\theta(s), \sigma_b))$
    \State \textbf{2. Evaluation:} $q_{(i)} \leftarrow Q_{\psi_\omega}(s, a''_{(i)})$
    \State \textbf{3. Compute target:} $\bar{a} \leftarrow \text{softmax} \: (q_{(i)})$ \eqref{eq:target_action}
    \State \textbf{4. Update:} Minimize $J(\theta) = ||\varphi_\theta(s) - \bar{a}||^2_2$ \eqref{eq:actor_update}
  \end{algorithmic}
\end{algorithm}

\textbf{Escape from local optima.}
DBU enables the actor to escape local optima using the target action construction: by perturbation and critic-based evaluation, it searches for high-value targets beyond the neighborhood $\calB(\hat a)$. The softmax average mitigates potential critic overestimation problems and smoothens the target. Even if the Q-function is non-convex in the action space, iteratively applying DBU amounts to a regression towards an identified global optimum, without requiring multi-worker approaches and value function truncation as in \citet{Jain2025}.

\textbf{Critic design.}
DBU is an off-policy algorithm that uses replay buffers, target Q-networks updated with polyak averaging, and double Q-learning. The critic $\psi_\omega$ receives a state $s$ and a feasible action $a$ as input and outputs $q = Q_{\psi_\omega}(s, a)$. For implementation details, we refer to App.~\ref{app:implementation}.

\subsection{Unified Treatment of Hybrid Spaces}

Our framework offers a unified handling of hybrid action vectors $a = (a_d, a_c)$. In parameterized spaces, the continuous parameter $a_c$ is often semantically coupled to the discrete choice $a_d$.
Traditional approaches typically either require separate network heads, which struggle to account for action dependencies, or hierarchical architectures, which can destabilize training.
In our approach, the actor outputs a joint proto-action $\hat{a} = (\hat{a}_d, \hat{a}_c) \in \calA'$. We handle exploration and updates uniformly across components: for exploration (SDN), we sample $a_d$ from the discrete neighborhood and $a_c$ from the continuous Gaussian neighborhood simultaneously, evaluating the joint tuple to capture interdependencies and making action selection as efficient as in discrete spaces.
For updates (DBU), the distance loss applies to the full vector: $J(\theta) = \|\hat{a}_d - \bar{a}_d\|^2 + \|\hat{a}_c - \bar{a}_c\|^2$. This allows the discrete and continuous components to inform each other through the shared representation layers of the actor, performing approximate joint coordinate ascent without combinatorial design overhead. Unlike hierarchical methods suffering from a regret floor if the high-level choice is sub-optimal (Prop.~\ref{prop:regret_appendix}), the unified distance loss allows mutual regularization of the continuous and discrete components.

\section{Theoretical Properties}\label{sec:analysis}
Having detailed DGRL's algorithmic structure, we analyze its theoretical properties as an action search and learning operator. We propose that i) stochastic sampling provides meaningful local coverage, ii) regression-based updates avoid variance growth with action cardinality, iii) we enable local value improvement in structured spaces, and iv) DGRL extends naturally to hybrid spaces.

\textbf{Assumptions.}
DGRL generally requires the same assumptions as Wolpertinger and DNC, i.e., an integer coordinate system, a continuous relaxation $\calA'$ of $\calA$, and a critic trained on valid discrete actions. When distances are uninformative with respect to value, e.g., in our Recommender environment, the consistency and variance results (App.~\ref{app:consistency}, Thm.~\ref{thm:variance}) continue to hold and DGRL is applicable, but local improvement is no longer guaranteed. Assumption~\ref{ass:latent_lipschitz} (Latent Lipschitz Continuity) is only required for Prop.~\ref{prop:lipschitz_continuity} and the sub-optimality bound of the executed action in Prop.~\ref{prop:improvement}.

\textbf{Efficiency of Volumetric Exploration.}
In high-dimensional discrete spaces, enumerating the neighborhood $\mathcal{B}(\hat{a})$ is intractable. SDN instead uses stochastic sampling to approximate $\max_{a \in \mathcal{B}(\hat{a})} Q(s, a)$. While consistency holds as $K \to \infty$ (App.~\ref{app:consistency}), practical efficiency depends on the geometry of the candidate region. We contrast SDN's volumetric exploration with the axial search used in prior methods such as DNC.

\begin{proposition}[Volumetric vs.\ Axial Candidate Sets]\label{prop:volumetric_supportexploration}
Let $L\in\mathbb{N}$ be the Chebyshev radius and $N$ the action dimension.
Define the SDN candidate set
$\mathcal{C}_{\mathrm{SDN}} := B_\infty(\hat a, L)\cap\mathbb{Z}^N$
and the axial-search candidate set
$\mathcal{C}_{\mathrm{axial}} := \bigl(\bigcup_{i=1}^{N}\{\hat a + \alpha e_i:\alpha\in[-L,L]\}\bigr)\cap\mathbb{Z}^N$.
For $N\ge 2$,$
|\mathcal{C}_{\mathrm{SDN}}| = (2L+1)^{N}, \qquad
|\mathcal{C}_{\mathrm{axial}}| \le 2NL + 1$, i.e., the SDN candidate set is exponential in $N$ while axial sets are
only linear in $N$. Equivalently, the underlying continuous candidate
\emph{generation regions} have Lebesgue measures
$\mu(B_\infty(\hat a,L)) = (2L)^{N}$ and $\mu(S_{\mathrm{axial}}) = 0$ for the union $S_{\mathrm{axial}}$ of the axial segments;
the exponential vs. linear cardinality gap is the discrete analogue of the dimensional difference between the generating regions.
\end{proposition}

\paragraph{Implication.}
SDN's discrete neighborhood spans an exponentially larger set of
multi-coordinate combinations than axial search. Off-axis discrete
optima -- which are inaccessible to coordinate-aligned exploration --
are captured by SDN with non-vanishing probability under the independent
coordinate sampling distribution. See App.~\ref{app:volumetric_supportexploration} for the full derivation.

\textbf{Variance-Independent Updates.}
The primary bottleneck in scaling model-free RL to large discrete spaces is gradient variance. Standard Policy Gradient (PG) methods rely on the score function estimator $\nabla \log \pi(a|s)$. As $|\calA| \to \infty$, the probability mass on any action $\pi(a|s) \to 0$, causing variance to scale linearly with $|\calA|$. DBU avoids this pathology.

\begin{theorem}[Removal of Action-Cardinality Dependence]\label{thm:variance}
Let $\varphi_\theta:\mathcal{S}\to\mathcal{A}'\subseteq\mathbb{R}^{N}$
be the actor and $J_\theta(s):=\partial\varphi_\theta(s)/\partial\theta$
its parameter Jacobian. Assume $\varphi_\theta$ is $G$-Lipschitz in
$\theta$ uniformly in $s$, so that $\|J_\theta(s)\|_{2}\le G$. Then for
every state $s$, the per-sample DBU gradient noise -- the covariance of
$g_{\mathrm{DBU}}$ over the randomness of the target $\bar a\mid s$ --
satisfies
\(
\operatorname{Tr}\!\bigl(\operatorname{Cov}[g_{\mathrm{DBU}}\mid s]\bigr)
\;\le\; G^{2}\,\operatorname{Tr}\!\bigl(\operatorname{Cov}[\bar a\mid s]\bigr).
\)
Both sides are independent of $|\mathcal{A}|$: the RHS is a
property of the target distribution $\bar a\mid s$, which is constructed
from a fixed number $M$ of candidates and lies in the relaxed action
space $\mathcal{A}'\subseteq\mathbb{R}^{N}$.
\end{theorem}

\textit{Proof Sketch.}
$g_{\mathrm{DBU}} = J_\theta(s)^{\top}\bigl(\varphi_\theta(s)-\bar a\bigr)$.
Conditioning on $s$, both $\varphi_\theta(s)$ and $J_\theta(s)$ are
deterministic constants; the only stochasticity is in $\bar a$. Since
covariance is invariant under deterministic shifts and linear under
deterministic linear maps,
$\operatorname{Cov}[g_{\mathrm{DBU}}\mid s]=J_\theta(s)^{\top}\operatorname{Cov}[\bar a\mid s]J_\theta(s)$.
Taking trace and using $\operatorname{Tr}(A^{\top}CA)\le\|A\|_{2}^{2}\operatorname{Tr}(C)$
for PSD $C$ gives the main result. Full proof in App.~\ref{app:action_cardinality}.

\textbf{Trust Region and Local Value Improvement.}
In many numerically or ordinally structured action spaces, it is reasonable to assume that distances between actions translate to distances between Q-values: e.g., in inventory replenishment, a small change in order quantities usually has a bounded effect on Q-values. Similarly, switching adjoining items in ordered lists often has bounded effects on the value function. We formalize this in the following, see App.~\ref{app:latent_lipschitz} for further empirical grounding.

\begin{assumption}[Latent Lipschitz Continuity]
\label{ass:latent_lipschitz}
We assume a discrete action space $\calA$ with an embedding function
$\phi: \calA \to \calZ \subseteq \mathbb{R}^n$ (where $\calZ$ corresponds to our relaxed space $\calA'$), such that the state-action value function $Q(s, a)$ admits a continuous extension $\tilde{Q}: \calS \times \calZ \to \mathbb{R}$ which is $L$-Lipschitz continuous with respect to a latent metric $\|\cdot\|_p: |\tilde{Q}(s, \phi(a)) - \tilde{Q}(s, \phi(a'))| \leq L \|\phi(a) - \phi(a')\|_p, \forall a, a' \in \calA$. Furthermore, we assume the embedding $\phi$ is semantic, such that proximity in $\calZ$ implies functional similarity in $\calA$.
\end{assumption}

Therefore, distance is a proxy for value in Lipschitz continuous structured action spaces:

\begin{proposition}[Approximation Bound via Lipschitz Continuity]
\label{prop:lipschitz_continuity}
Let $Q(s, \cdot): \mathcal{A}' \to \mathbb{R}$ be $L_Q$-Lipschitz continuous with respect to the Euclidean norm $\|\cdot\|_2$. Let $a^\star$ be the optimal discrete action, $\hat{a}$ a continuous proto-action, and $\bar{a}$ a target action. Let $a_{\mathrm{nn}}$ denote the nearest discrete neighbor of $\hat{a}$. Then
\(
Q(s, a^\star) - Q(s, a_{\mathrm{nn}})
\leq L_Q \left( \sqrt{J(\theta)} + \|\bar{a} - a^\star\|_2 + \varepsilon_{\mathrm{round}} \right),
\)
where $J(\theta) = \|\hat{a} - \bar{a}\|_2^2$ and $\varepsilon_{\mathrm{round}} = \|\hat{a} - a_{\mathrm{nn}}\|_2$.
\end{proposition}

\textit{Proof Sketch.} The value gap decomposes into the distance from $\hat{a}$ to the target $\bar{a}$, and from $\bar{a}$ to the true optimum $a^\star$. By Lipschitz continuity, minimizing the distance $J(\theta)$ minimizes the upper bound of the value loss. See App.~\ref{app:lipschitz_continuity}.

We frame DBU as an instance of Generalized Policy Iteration, in which the target $\bar{a}$ is constructed via a softmax operation that filters actions with Q-values above the current average.

\begin{proposition}[Local Improvement under Smoothness]\label{prop:improvement}
Assume that, on a neighborhood of $\mu_{\mathrm{old}}:=\varphi_{\theta_{\mathrm{old}}}(s)$,
(i) $Q(s,\cdot)$ is $\beta$-smooth and twice continuously differentiable; (ii) the perturbation scale $\sigma_b$ used to construct
$\bar a$ is small enough that the first-order Taylor expansion of $Q$
in the perturbations is valid, and $\nabla_a Q(s,\mu_{\mathrm{old}})\neq 0$; (iii) the actor learning rate $\alpha$ is sufficiently small. Then a single DBU step on the actor satisfies $Q\bigl(s,\varphi_{\theta_{\mathrm{new}}}(s)\bigr)
\;\ge\;
Q\bigl(s,\varphi_{\theta_{\mathrm{old}}}(s)\bigr)
\;+\;\Omega(\alpha\sigma_b^{2}/\tau_b)\cdot\bigl\|J_\theta(s)^{\top}\nabla_aQ(s,\mu_{\mathrm{old}})\bigr\|_{2}^{2}$. Combined with Lipschitz continuity of $Q$ on $\mathcal{A}'$
(Assumption~\ref{ass:latent_lipschitz}), this implies $
Q\bigl(s,a^{\star}\bigr) - Q\bigl(s,a_{\mathrm{exec}}\bigr)
\;\le\; L_Q\bigl(\sqrt{J(\theta_{\mathrm{new}})} + d(\bar a,a^{\star}) + \varepsilon_{\mathrm{round}}\bigr)$,
i.e., DBU improvement at the proto-action transfers to a bounded
sub-optimality gap for the executed discrete action $a_{\mathrm{exec}}$.
\end{proposition}

In environments where the action-value function is locally smooth, DBU thus promotes local improvement of the continuous surrogate $Q(s, \varphi_\theta(s))$. In settings where this assumption is violated, the update remains well-defined and the variance and stability properties established continue to hold; DBU maintains full applicability. See App.~\ref{app:improvement} for a rigorous derivation.

\textbf{Extension to Hybrid Spaces.}
Unlike hierarchical architectures that suffer from a commitment bottleneck, our unified update allows for mutual regularization by selecting a joint discrete-continuous proto-action. This eliminates the hard regret floor typical for cluster-based methods (Prop.~\ref{prop:regret_appendix}), as the policy gradient flows through a coupled representation rather than a sequential decision tree. DGRL treats the latent space $\mathcal{Z}$ as a unified manifold, where the joint proto-action $(\hat a_d, \hat a_c)$ provides structural constraints while SDN-based selection refines the local manifold. This creates mutual regularization, enabling stable, joint optimization that thrives even under partial coupling (Lemma~\ref{lem:hybrid_appendix}).

\begin{remark}[Approximate Coordinate Ascent]
For a hybrid action $a=(a_d, a_c)$, the DBU loss decomposes into orthogonal components: $\|\hat{a} - \bar{a}\|^2 = \|\hat{a}_d - \bar{a}_d\|^2 + \|\hat{a}_c - \bar{a}_c\|^2$. Optimizing this joint loss can be interpreted as an approximate block coordinate update with respect to a local surrogate of the Q-function. Crucially, while the \textit{gradients} are decoupled, the \textit{target} $(\bar{a}_d, \bar{a}_c)$ is derived from a joint search, capturing dependencies without the interference typical of joint gradient descent.
\end{remark}

\section{Numerical Study}\label{sec:exp_setup}

\textbf{Experimental Design.}
We test the empirical performance of DGRL on twelve versions of four environments common in the literature on scalable DRL, covering a wide range of (irregularly) structured discrete and hybrid problem settings. First, we consider a 2D maze environment \citep[cf.][]{dulac2015, chandak2019, akkerman2024dynamic}: an agent navigates a maze by selecting $N$ out of $D$ evenly spaced actuators. The actuators are logically sequenced (randomly shuffled) in the (irregularly) structured versions. The step size is parametrized separately in the hybrid versions. Second, we consider a categorical Recommender environment \citep[cf.][]{dulac2015, chandak2019}, based on real-world movie preference data \citep{Harper2015}: an agent recommends $N$ out of 343 movies, the customer can choose a movie or terminate the episode. In the hybrid variants, the price per movie is estimated separately. Finally, we consider a structured Job Shop Scheduling Problem and a structured Joint Inventory Replenishment Problem \citep[cf.][]{akkerman2024dynamic, vanvuchelen2022}. For environment details, we refer to App.~\ref{sec:environments}.

We compare the performance of DGRL to six discrete and three hybrid benchmarks: as discrete benchmarks, we use a Vanilla Actor-Critic (VAC) algorithm, Continuous Actor-Critic Learning Automaton (Cacla) \citep{hasselt2009}, Wolpertinger \citep{dulac2015}, Learned Action Representations (LAR) \citep{chandak2019} and Dynamic Neighborhood Construction (DNC) with simulated annealing and greedy search \citep{akkerman2024dynamic}. Due to the lack of scalable hybrid benchmarks, we extend two common benchmarks to the continuous-to-discrete paradigm: H-Cacla and H-DNC. In addition, we use HyAR, a hybrid version of LAR \citep{Li2022}. While DGRL, Cacla, DNC, H-Cacla, and H-DNC scale to all action space sizes, the other benchmarks are restricted by exponentially increasing memory requirements. For details, we refer to App.~\ref{sec:baselines}. All results are medians over 10 training seeds, with shading showing the interquartile range.

\textbf{Empirical Performance and Scalability.} DGRL demonstrates superior performance across test domains, consistently achieving higher final rewards with faster convergence and lower variance than the benchmarks. These advantages are especially pronounced in large-scale and irregularly structured environments, where limitations of grid-based or exhaustive search methods become most visible.

\begin{figure}[th]
    \includegraphics[width=0.261\textwidth]{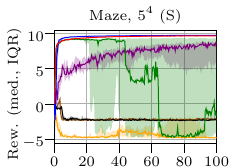}%
    \hfill%
    \includegraphics[width=0.239\textwidth]{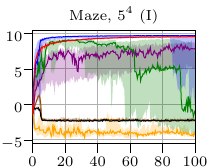}%
    \hfill%
    \includegraphics[width=0.261\textwidth]{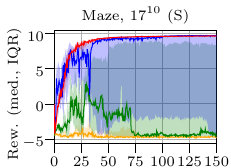}%
    \hfill%
    \includegraphics[width=0.239\textwidth]{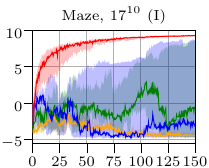}%
    \hfill%
    \par\smallskip
    \includegraphics[width=0.263\textwidth]{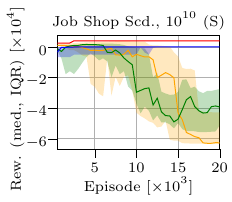}%
    \hfill%
    \includegraphics[width=0.237\textwidth]{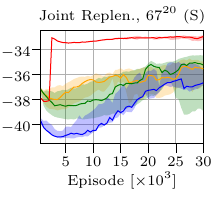}%
    \hfill%
    \includegraphics[width=0.265\textwidth]{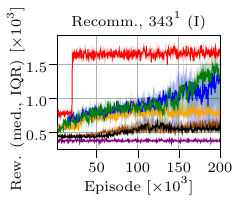}%
    \hfill%
    \includegraphics[width=0.235\textwidth]{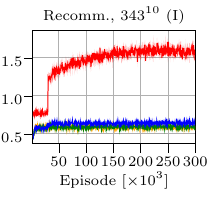}%
    \hfill%
    \par\smallskip
    \includegraphics[width=\textwidth]{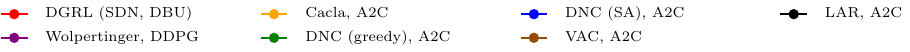}
    \vspace{-6mm}
    \caption{Results for discrete environments, median over 10 random seeds. Titles indicate size and type (structured or irregular) of action space. Legend indicates mapping method and RL algorithm.}
    \label{fig:results_discrete}
    \vspace{-6mm}
\end{figure}

\textbf{Discrete Action Spaces.}
We evaluate DGRL across discrete configurations with both structured (S) and irregularly structured (I) action spaces (Fig.~\ref{fig:results_discrete}). While DGRL and DNC perform competitively in small-scale mazes, the performance of DNC collapses in high-dimensional, irregularly structured spaces. While DNC is restricted to coordinate-aligned search, DGRL’s volumetric exploration and regression-based updates maintain stability, outperforming DNC by up to 66\% and validating Prop.~\ref{prop:volumetric_supportexploration} and~\ref{thm:variance}. This scalability gap widens in the Recommender environment: while DNC succeeds in small settings, it fails to scale, allowing DGRL to outperform it by 147\%. Notably, DGRL maintains superior sampling density in irregularly structured environments, validating the $L_\infty$ search volume.

Across both structured and irregular spaces, DGRL exhibits consistent convergence behavior. In contrast, benchmarks like Wolpertinger and DNC suffer from significant variance and performance degradation. Wolpertinger displays performance gaps to DGRL already for small-scale environments, in addition to higher training variance. Wolpertinger’s failure stems from inefficient exploration in irregular manifolds and DDPG’s susceptibility to local optima, both of which DGRL mitigates through its volumetric exploration and denoised regression targets. Finally, methods like Wolpertinger and LAR require an apriori lookup table, which is computationally prohibitive for $\geq 10^{10}$ actions, whereas DGRL scales without explicit enumeration.

\textbf{Hybrid Action Spaces.}
As shown in Fig.~\ref{fig:results_hybrid}, DGRL maintains its performance advantage in hybrid domains, outperforming the best-performing benchmarks by 11\% to 129\%. As in discrete environments, these performance gaps become especially visible in large-scale and irregularly structured environments. Our results validate the theoretical benefits of the joint manifold search: whereas hierarchical baselines like H-DNC are restricted by low-level choices conditioned on fixed, potentially sub-optimal high-level parameters, DGRL performs a coordinated selection of both action components. This joint optimization, combined with DBU's lack of a hard regret floor (Prop.~\ref{prop:regret_appendix}), prevents high-level selection errors from propagating into and biasing the low-level policy. By eliminating hierarchical dependencies while maintaining coordinated action selection, DGRL achieves a level of scalability and stability out of reach for traditional decomposed architectures.

\textbf{Computational Complexity.}
We evaluate average step times across varying action space cardinalities in Table~\ref{tab:results_compute}. While the inference times of Wolpertinger and LAR are initially lower due to their reliance on pre-computed lookup tables, these methods become computationally and memory-prohibitive as the action space expands. In contrast, while DNC's step times scale poorly with action space dimensionality, DGRL exhibits near-constant time complexity relative to the action space size. By leveraging independent sampling and parallelized local searches, DGRL maintains efficient neighborhood evaluation regardless of cardinality, offering a scalable alternative to exhaustive or grid-constrained search paradigms.

Learning times of DGRL grow linearly with exponentially growing action spaces, as can be seen in Fig.~\ref{fig:results_discrete}: between the small and large Maze and the small and large Recommender, it takes between 2 and 4 times as many episodes to reach similar performance levels. The inherent instability of the baselines prevents a direct comparison; in the structured Maze environment, where DNC converges, its learning times are similar to DGRL.

\begin{figure}[th]
    \includegraphics[width=0.254\textwidth]{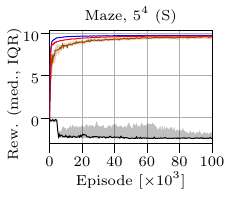}%
    \hfill%
    \includegraphics[width=0.234\textwidth]{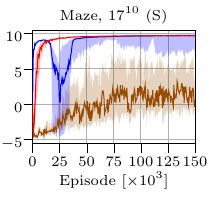}%
    \hfill%
    \includegraphics[width=0.272\textwidth]{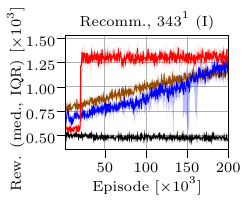}%
    \hfill%
    \includegraphics[width=0.234\textwidth]{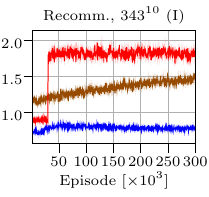}%
    \hfill%
    \par\smallskip
    \includegraphics[width=\textwidth]{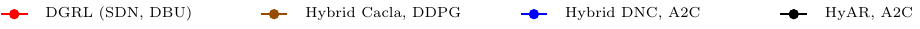}
    \vspace{-6mm}
    \caption{Results for hybrid environments, median over 10 random seeds. Titles indicate size and type (structured or irregular) of action space. Legend indicates mapping method and RL algorithm.}
    \label{fig:results_hybrid}
    \vspace{-0mm}
\end{figure}

\begin{table}[th]
    \centering
    \vspace{0mm}
    \caption{Step times of algorithms in millisec. Averaged over 1000 episodes in the Maze environment. Dashes denote infeasible settings due to memory errors.}
    \label{tab:results_compute}
    \footnotesize
    \begin{tabular}{r|c|c|c|c|c}
        $|\calA|$ & DGRL & DNC (SA) & Cacla & Wolp. & LAR \\
        \hline
        $5^{5}$ & 3.3 & 12.8 & 0.1 & 1.1 & 0.2 \\
        $10^{10}$ & 3.4 & 19.4 & 0.1 & - & - \\
        $20^{20}$ & 3.4 & 33.6 & 0.1 & - & - \\
        $50^{50}$ & 3.2 & 80.2 & 0.1 & - & -
    \end{tabular}
    \vspace{-2mm}
\end{table}

\textbf{Ablation Studies and Metric Sensitivity.}
We isolate the contributions of SDN and DBU to evaluate their individual impacts on stability and exploration (App.~\ref{app:results}, Fig.~\ref{fig:results_ablation}). We find that full DGRL outperforms all other algorithmic setups, highlighting the synergies between SDN and DBU. SDN drives exhaustive local exploration, covering the action manifold more uniformly than DNC's axial search. DBU provides a global, low-variance learning signal that stabilizes updates relative to standard policy gradients (A2C) and prevents local optima entrapment common in DDPG. On the contrary, combining DBU and Wolpertinger is harmful -- DBU complements volumetric sampling rather than $k$-NN lookup. These results highlight the inherent links between SDN and DBU: SDN performs stochastic action selection based on distances, while DBU minimizes distances to high-value target regions, creating a volumetric projection loop.

Furthermore, evaluating different distance metrics confirms the strategic advantage of the Chebyshev distance ($L_\infty$). While Euclidean and Manhattan metrics are viable, they exhibit higher variance and require dimension-specific tuning of the radius $L$. Empirically underscoring the dimensional invariance of its search volume (Prop.~\ref{prop:volumetric_supportexploration}), our results validate the $L_\infty$ hypercube as the optimal geometry for distance-guided exploration and inference.

We test the post-training generalization performance of DGRL by testing a trained model on altered Maze problems (App.~\ref{app:results}, Tab.~\ref{tab:results_generalization}). By design, continuous-discrete algorithms cannot tolerate an increase of actor outputs $N$ or shuffled options in $D$. In contrast, shuffling actor outputs does not affect their performance. When adding options in the structured Maze, DGRL maintains performance longer than DNC, only breaking when adding more options. This shows that DGRL is stable under structure-maintaining alterations of the environment, while breaking when environments deviate from the learned structure too much.

\section{Conclusion}\label{sec:conclusion}
We propose Distance-Guided Reinforcement Learning (DGRL), scaling RL to large discrete and hybrid action spaces with up to $10^{36}$ candidates. DGRL encompasses Sampled Dynamic Neighborhoods, enabling volumetric stochastic refinement of the continuous proto-action using Chebyshev constraints, and Distance-Based Updates, transforming the actor update to a stable regression task and decoupling gradient variance from action space cardinality. Our theoretical results show the dimensional invariance of Chebyshev-based sampling and convergence of Distance-Based Updates in structured domains. Empirically, DGRL outperforms state-of-the-art benchmarks by up to 66\% across multiple (irregularly) structured environments, while displaying faster and more stable convergence. Beyond these results, DGRL offers a foundation for metric-aware RL, suggesting that leveraging problem structure and distance-guided sampling can replace categorical representations.


\newpage

\begin{ack}
We thank the BAIS research group at TUM for valuable comments and discussions. The work of Heiko Hoppe was supported by the Munich Data Science Institute with a Linde/MDSI PhD Fellowship. Wouter van Heeswijk acknowledges partial support from Marie Skłodowska-Curie Grant Agreement No. 101119635. Fabian Akkerman was funded via TKI Dinalog Grant 2024-1-363-TKI. Maximilian Schiffer was funded by TUM.
\end{ack}

\bibliographystyle{plainnat}
\bibliography{DGRL_bib}


\newpage

\appendix

\section{Theoretical Analysis of DGRL}
\label{app:theory}

\subsection{Dimensional Invariance of Chebyshev Neighborhoods}
\label{app:chebyshev_invariance}

\begin{proposition}[Dimensional Invariance of Chebyshev Neighborhoods]
\label{prop:chebyshev_invariance}
Consider a search neighborhood $\mathcal{N}$ constructed by independent sampling within a range $\pm L$ across $N$ dimensions. The geometry of $\mathcal{N}$ corresponds to a ball under the $L_\infty$ (Chebyshev) metric. While the $L_2$ radius required to encompass the support corners grows as $\mathcal{O}(L\sqrt{N})$, the $L_\infty$ radius remains invariant to $N$. This ensures the trust region remains tightly coupled to the sampled support, preventing the volume expansion inherent to Euclidean metrics in high dimensions.
\end{proposition}

\begin{proof}
Let the search neighborhood $\mathcal{N} \subset \mathbb{R}^N$ be defined by independent sampling in each dimension $i \in \{1, \dots, N\}$ within a local range $[-L, L]$ relative to a center point $\hat a$. A point $a \in \mathcal{N}$ can be represented as $a = \hat a + \delta$, where $\delta \in [-L, L]^N$.

By definition, the $L_p$ norm of the displacement vector $\delta$ is:
\[ \|\delta\|_p = \left( \sum_{i=1}^N |\delta_i|^p \right)^{1/p} \]

\textbf{1. The $L_2$ case (Euclidean):}
Consider the corner points of the neighborhood where $|\delta_i| = L$ for all $i$. The Euclidean distance from the center to these points is:
\[ \|\delta\|_2 = \sqrt{\sum_{i=1}^N L^2} = \sqrt{N L^2} = L \sqrt{N} \]
As the dimensionality $N$ increases, the radius $R_2 = L \sqrt{N}$ required to cover the sampled neighborhood grows at a rate of $\mathcal{O}(\sqrt{N})$. If we fix the radius to be independent of $N$, the fraction of the hypercube volume captured by the $L_2$ ball vanishes as $( \frac{\pi e}{2N} )^{N/2}$, leading to the volume expansion problem where the policy update is informed by regions without sampled support.

\textbf{2. The $L_\infty$ case (Chebyshev):}
The $L_\infty$ norm is defined as the limit of the $p$-norm as $p \to \infty$:
\[ \|\delta\|_\infty = \max_{i} |\delta_i| \]
For any point $a$ in the sampled neighborhood $\mathcal{N}$, by construction $|\delta_i| \leq L$. Therefore:
\[ \sup_{a \in \mathcal{N}} \|a - \hat a\|_\infty = L \]
The radius $R_\infty = L$ is strictly invariant to $N$. 

\textbf{3. Conclusion on Trust Region Coupling:}
Because the $L_\infty$ ball is geometrically congruent to the hypercubic sampling distribution, the trust region defined by $B_\infty(\hat a, L)$ maintains a constant density of support relative to its volume. This allows Sampled Dynamic Neighborhoods (SDN) to maintain a stable Lipschitz bound on the action selection in structured problems regardless of the cardinality or dimensionality of the action space (see App.~\ref{app:latent_lipschitz} for the Lipschitz assumption in structured action spaces).
\end{proof}

\subsection{Consistency of Sampled Neighborhoods}
\label{app:consistency}

SDN approximates the maximization of the action-value function over a discrete neighborhood
$\mathcal{B}(\hat a)$ by Monte Carlo sampling. We show that this procedure yields a consistent estimator
of the neighborhood value.

\begin{proposition}[Consistency of SDN Estimator]
\label{thm:consistency_appendix}
Let $\mathcal{B}(\hat a) \subset \mathcal{A}$ be a finite neighborhood of candidate actions.
Let $P(\cdot \mid \hat a)$ be a sampling distribution over $\mathcal{B}(\hat a)$ with full support, i.e.,
\[
P(a \mid \hat a) > 0 \quad \forall a \in \mathcal{B}(\hat a).
\]
Let $a^\star_{\mathrm{local}} := \arg\max_{a \in \mathcal{B}(\hat a)} Q(s,a)$ and define the SDN estimator
\[
\hat Q_{\mathrm{SDN}}(s,\hat a) := \max_{1 \le k \le K} Q(s,a'_k),
\quad a'_k \sim P(\cdot \mid \hat a) \ \text{i.i.d.}
\]
Then $\hat Q_{\mathrm{SDN}}(s,\hat a)$ is a consistent estimator of
$\max_{a \in \mathcal{B}(\hat a)} Q(s,a)$, i.e.,
\[
\hat Q_{\mathrm{SDN}}(s,\hat a)
\;\xrightarrow{a.s.}\;
\max_{a \in \mathcal{B}(\hat a)} Q(s,a)
\quad \text{as } K \to \infty.
\]
\end{proposition}

\begin{proof}
Let $a^\star_{\mathrm{local}} := \arg\max_{a \in \mathcal{B}(\hat a)} Q(s,a)$.
By the full-support assumption, the probability of sampling $a^\star_{\mathrm{local}}$ in a single draw is
\[
p := P(a^\star_{\mathrm{local}} \mid \hat a) > 0.
\]
Since samples $\{a'_k\}_{k=1}^K$ are drawn independently, the probability that
$a^\star_{\mathrm{local}}$ is not sampled in $K$ draws is $(1-p)^K$.
Therefore,
\[
\mathbb{P}\!\left(
\hat Q_{\mathrm{SDN}}(s,\hat a) \neq Q(s,a^\star_{\mathrm{local}})
\right)
= (1-p)^K.
\]
As $K \to \infty$, we have $(1-p)^K \to 0$, implying
\[
\mathbb{P}\!\left(
\hat Q_{\mathrm{SDN}}(s,\hat a) = Q(s,a^\star_{\mathrm{local}})
\right)
\to 1.
\]
Hence, $\hat Q_{\mathrm{SDN}}(s,\hat a)$ converges almost surely to
$\max_{a \in \mathcal{B}(\hat a)} Q(s,a)$ as $K \to \infty$.
\end{proof}

\begin{corollary}[Finite-Sample Guarantee]
For any $\delta \in (0,1)$, if $K \ge \frac{\log(1/\delta)}{p},$ then $\mathbb{P}\!\left(\hat Q_{\mathrm{SDN}}(s,\hat a) = \max_{a \in \mathcal{B}(\hat a)} Q(s,a)\right) \ge 1 - \delta.$
\end{corollary}

\subsection{Geometric Support and Exploration Coverage}
\label{app:volumetric_supportexploration}

This section proves the cardinality gap between SDN and axial candidate sets (Prop.~\ref{prop:volumetric_supportexploration}) and discusses what it implies for a fixed sampling budget $K$.

\begin{proposition}[Volumetric vs.\ Axial Candidate Sets]\label{prop:volumetric_supportexploration_app}
Let $L\in\mathbb{N}$ be the Chebyshev radius and $N$ the action dimension.
Define the SDN candidate set
$\mathcal{C}_{\mathrm{SDN}} := B_\infty(\hat a, L)\cap\mathbb{Z}^N$
and the axial-search candidate set
$\mathcal{C}_{\mathrm{axial}} := \bigl(\bigcup_{i=1}^{N}\{\hat a + \alpha e_i:\alpha\in[-L,L]\}\bigr)\cap\mathbb{Z}^N$.
For $N\ge 2$, $|\mathcal{C}_{\mathrm{SDN}}| = (2L+1)^{N}, \qquad
|\mathcal{C}_{\mathrm{axial}}| \le 2NL + 1$,i.e., the SDN candidate set is exponential in $N$ while axial sets are
only linear in $N$. Equivalently, the underlying continuous candidate
\emph{generation regions} have Lebesgue measures
$\mu(B_\infty(\hat a,L)) = (2L)^{N}$ and $\mu(S_{\mathrm{axial}}) = 0$ for the union $S_{\mathrm{axial}}$ of the axial segments;
the exponential vs. linear cardinality gap is the discrete analogue of the dimensional difference between the generating regions.
\end{proposition}

\begin{proof}
\textbf{(i) Discrete cardinalities.}
$B_\infty(\hat a,L)\cap\mathbb{Z}^N
=\prod_{i=1}^{N}\bigl(\{\hat a_i\}+\{-L,\dots,L\}\bigr)$
contains $(2L+1)^N$ integer points, growing exponentially in $N$. Each
axial line segment $\{\hat a + \alpha e_i:\alpha\in[-L,L]\}\cap\mathbb{Z}^N$
contributes at most $2L+1$ integer points, and the union over $i=1,\dots,N$
contains at most $N(2L+1) - (N-1) = 2NL+1$ distinct points (subtracting
the shared center $\hat a$). Hence
$|\mathcal{C}_{\mathrm{axial}}|\le 2NL+1$.

\textbf{(ii) Continuous generation regions.}
$B_\infty(\hat a,L)$ is a hypercube of side length $2L$, so
$\mu(B_\infty(\hat a,L)) = (2L)^{N}>0$. Each axial line $\ell_i = \{\hat a+\alpha e_i\}$
is the image of an interval under an injective affine map into $\mathbb{R}^N$
with rank-$1$ Jacobian; for $N\ge 2$, $\mu(\ell_i)=0$ and by countable
additivity $\mu\bigl(\bigcup_i \ell_i\bigr)=0$.

\textbf{(iii) Why the continuous gap matters in the discrete setting.}
Both candidate sets are finite and hence have zero Lebesgue measure on
$\mathbb{R}^N$; the operational distinction between SDN and axial search
is therefore not about the measure of the \emph{evaluated} candidates,
but about the measure of the \emph{generation region} from which they
are drawn. SDN's generation region is full-dimensional, so its discrete
candidates can carry arbitrary off-axis directions; axial generation is
$1$-dimensional in $\mathbb{R}^N$, so its discrete candidates are
restricted to coordinate axes. The cardinality gap \(
|\mathcal{C}_{\mathrm{SDN}}| = (2L+1)^{N}, \qquad
|\mathcal{C}_{\mathrm{axial}}| \le 2NL + 1,
\) is the discrete consequence of this dimensional
asymmetry.
\end{proof}

\begin{remark}[Sampling distribution]
Each coordinate $n$ of a candidate $a'$ is sampled independently with probability proportional to $L-|\hat a_n-a'_n|+\tau_s$ (linear scheme) or $\exp(-|\hat a_n-a'_n|/\tau_s)$ (softmax scheme).
Both are positive for every integer in $[\hat a_n-L,\hat a_n+L]$, so every $a'\in\mathcal{C}_{\mathrm{SDN}}$ has probability $p(a')=\prod_{n=1}^{N}p(a'_n)>0$, i.e., all $(2L+1)^N$ multi-coordinate combinations receive non-vanishing sampling probability.
\end{remark}

\paragraph{Marginal versus joint coverage.}
Scaling $K \approx L \cdot N$ (App.~\ref{app:hyperparams}) preserves the expected coverage of each coordinate, but not the probability of sampling a specific multi-coordinate action jointly. Under independent coordinate sampling, a single draw matches a target $a^\star \in \mathcal{C}_{\mathrm{SDN}}$ only if all $N$ coordinates land on the required values simultaneously, giving the joint probability $p = \prod_{n=1}^{N} p(a^\star_n)$. Each factor is roughly constant in $N$, since the per-coordinate sampling distribution is set by the fixed radius $L$ (App.~\ref{app:chebyshev_invariance}) rather than by $N$. Calling this constant $q<1$, the product decays as $p \approx q^{N}$. Over $K$ draws, the hitting probability $1-(1-p)^{K} \approx Kp$ becomes $L N q^{N}$ under $K \approx L \cdot N$, which vanishes as $N \to \infty$ for any fixed $L$: polynomial growth of $K$ cannot offset exponential decay of $p$. Avoiding this would require $K = \Theta(q^{-N})$, i.e., exponential in $N$. This is exactly the candidate-set cardinality $(2L+1)^N$ of Prop.~\ref{prop:volumetric_supportexploration_app} restated as a sampling cost: hitting one specific point of a set costs a budget of that set's size, which is what moving away from exhaustive search is meant to avoid. Consistently, the finite-sample guarantee $K \geq \log(1/\delta)/p$ of App.~\ref{app:consistency} is exponential in $N$.

\paragraph{Recovery via DBU.}
The joint recovery of $a^\star$ is not left to SDN's random sampling but to DBU. By Eq.~\eqref{eq:es-alignment} in App.~\ref{app:improvement}, $\bar a - \mu_{\mathrm{old}} \approx \tfrac{\sigma_b^{2}}{\tau_b}\nabla_a Q(s,\mu_{\mathrm{old}})$, so each DBU step moves all $N$ coordinates of the proto-action jointly along the gradient of $Q$, instead of perturbing them independently at random. Repeated over training, this lets the proto-action converge towards the joint optimum $a^\star$ in every coordinate simultaneously, to the extent that the continuous extension of $Q$ is locally smooth (Assumption~\ref{ass:latent_lipschitz}).

\paragraph{SDN at execution.}
The joint-sampling difficulty does not reappear at execution, because SDN does not search for $a^\star$ from an uninformed starting point. Once $\hat a = \varphi_\theta(s)$ has converged near $a^\star$ in every coordinate, its nearest-neighbor rounding, always included in SDN's candidate set (App.~\ref{app:implementation}), lies at or immediately adjacent to $a^\star$, instead of needing to be located by chance. The $K$ additional SDN samples then serve a narrower purpose: correcting residual rounding error and covering local non-convexities of $Q$ near $\hat a$. This task is governed by the local geometry of $Q$ within radius $L$ rather than by the combinatorial size of the full neighborhood, which is why $K \approx L \cdot N$ suffices even though it would not suffice for unguided joint recovery.

\paragraph{Empirical consistency.}
Job Shop Scheduling ($N=10$), Joint Replenishment ($N=20$), and Recommender ($N=10$) all lie at dimensionalities where the per-round joint recovery probability derived above is already negligible under $K \approx L \cdot N$. Nevertheless, DGRL performs well, achieving median peak rewards of $4033$, $-328{,}336$, and $1942$ against $0.58$, $-362{,}376$, and $784$ for DNC (SA) (Tables~\ref{tab:disc_res1} and~\ref{tab:disc_res2}). This is consistent with performance being carried by DBU's gradient-based convergence of the proto-action rather than by SDN's per-round combinatorial coverage.

\subsection{Removal of Action-Cardinality Dependence}
\label{app:action_cardinality}

\begin{theorem}[Removal of Action-Cardinality Dependence]\label{thm:variance_app}
Let $\varphi_\theta:\mathcal{S}\to\mathcal{A}'\subseteq\mathbb{R}^{N}$
be the actor and $J_\theta(s):=\partial\varphi_\theta(s)/\partial\theta$
its parameter Jacobian. Assume $\varphi_\theta$ is $G$-Lipschitz in
$\theta$ uniformly in $s$, so that $\|J_\theta(s)\|_{2}\le G$. Then for
every state $s$, the per-sample DBU gradient noise -- the covariance of
$g_{\mathrm{DBU}}$ over the randomness of the target $\bar a\mid s$ --
satisfies
\(
\operatorname{Tr}\!\bigl(\operatorname{Cov}[g_{\mathrm{DBU}}\mid s]\bigr)
\;\le\; G^{2}\,\operatorname{Tr}\!\bigl(\operatorname{Cov}[\bar a\mid s]\bigr).
\)
Both sides are independent of $|\mathcal{A}|$: the RHS is a
property of the target distribution $\bar a\mid s$, which is constructed
from a fixed number $M$ of candidates and lies in the relaxed action
space $\mathcal{A}'\subseteq\mathbb{R}^{N}$.
\end{theorem}

\begin{proof}
The DBU stochastic gradient for a single transition is
\begin{equation}
g_{\mathrm{DBU}}\;=\;J_\theta(s)^{\top}\bigl(\varphi_\theta(s)-\bar a\bigr),
\qquad J_\theta(s):=\frac{\partial\,\varphi_\theta(s)}{\partial\theta}.
\label{eq:gdbu-def}
\end{equation}

\paragraph{Conditioning on $s$.}
We compute the covariance over the randomness of the target $\bar a$ at
a fixed state $s$, i.e., we consider the per-sample gradient noise due to
the randomness of the target. This isolates the dependence on
$|\mathcal{A}|$ from the effect of the state distribution on SGD variance.

Given $s$, $\varphi_\theta(s)$ and $J_\theta(s)$ are deterministic. For
any random vector $X$, deterministic vector $c$, and deterministic
matrix $A$,
$\operatorname{Cov}[A^{\top}(c-X)] = A^{\top}\operatorname{Cov}[X]A$.
Applying this with $X=\bar a$, $c=\varphi_\theta(s)$, $A=J_\theta(s)$
gives
$\operatorname{Cov}[g_{\mathrm{DBU}}\mid s]=J_\theta(s)^{\top}\operatorname{Cov}[\bar a\mid s]J_\theta(s)$.

\paragraph{Trace bound.}
For PSD $C$ and any matrix $A$,
\begin{equation}
\operatorname{Tr}(A^{\top}C A) \;=\; \operatorname{Tr}(C\,A A^{\top})
\;\le\; \|A A^{\top}\|_{\mathrm{op}}\operatorname{Tr}(C)
\;=\; \|A\|_{2}^{2}\operatorname{Tr}(C),
\label{eq:trace-ineq}
\end{equation}

where $\|\cdot\|_{\mathrm{op}}$ denotes operator norm and we used cyclicity
of trace and $\operatorname{Tr}(MN)\le\|M\|_{\mathrm{op}}\operatorname{Tr}(N)$
for PSD $M,N$. Apply \eqref{eq:trace-ineq} to \eqref{eq:gdbu-def} with
$A=J_\theta(s)$, $C=\operatorname{Cov}[\bar a\mid s]$, and use
$\|J_\theta(s)\|_{2}\le G$ (which follows from $G$-Lipschitzness of
$\varphi_\theta$ in $\theta$). This yields the trace bound $\operatorname{Tr}\!\bigl(\operatorname{Cov}[g_{\mathrm{DBU}}\mid s]\bigr)
\;\le\; G^{2}\operatorname{Tr}\!\bigl(\operatorname{Cov}[\bar a\mid s]\bigr)$.

\paragraph{Absence of the bias term.}
The covariance route makes no use of the uncentered second moment
$\mathbb{E}\|\varphi_\theta(s)-\bar a\|^{2}$. Had we instead bounded
$\operatorname{Tr}(\operatorname{Cov}[g_{\mathrm{DBU}}\mid s])
\le\mathbb{E}[\|g_{\mathrm{DBU}}\|^2\mid s]\le G^{2}\mathbb{E}[\|\varphi_\theta(s)-\bar a\|^2\mid s]$,
the bias--variance decomposition would have produced the additional
term $\|\varphi_\theta(s)-\mathbb{E}[\bar a\mid s]\|^{2}$, which is in
general non-zero away from the regression fixed point. The covariance
identity above does not produce this term, because covariance is
invariant under deterministic shifts.

\paragraph{Cardinality independence.}
The bound $\operatorname{Tr}\!\bigl(\operatorname{Cov}[g_{\mathrm{DBU}}\mid s]\bigr)
\;\le\; G^{2}\operatorname{Tr}\!\bigl(\operatorname{Cov}[\bar a\mid s]\bigr)$ depends only on $G$ and $\operatorname{Cov}[\bar a\mid s]$.
The target $\bar a$ is the softmax-weighted average of $M$ candidate
actions in $\mathcal{A}'\subseteq\mathbb{R}^{N}$, where $M$ is a fixed
hyperparameter (Sec.~\ref{sec:dbu}; $M=40$ in our experiments). Hence
$\operatorname{Tr}(\operatorname{Cov}[\bar a\mid s])\le\operatorname{diam}(\mathcal{A}')^{2}$,
which depends on the embedding dimension $N$ but not on $|\mathcal{A}|$.
\end{proof}

\paragraph{Comparison with score-function policy gradients.}
The standard PG estimator
$g_{\mathrm{PG}}=\nabla_\theta\!\log\pi_\theta(a\mid s)\,(Q(s,a)-V(s))$
has per-sample noise that diverges as $|\mathcal{A}|\to\infty$: when
probability mass spreads over many actions, $\pi(a\mid s)\to 0$ and
$\|\nabla\log\pi(a\mid s)\|^{2}=\Omega(|\mathcal{A}|)$.
Theorem~\ref{thm:variance_app} shows DBU exhibits no such scaling
-- its per-sample noise is governed by the geometry of
$\mathcal{A}'\subseteq\mathbb{R}^{N}$ alone.

\paragraph{Unconditional gradient noise.}
Marginalizing over $s$ adds an inter-state term
$\operatorname{Cov}_s\!\bigl[\mathbb{E}[g_{\mathrm{DBU}}\mid s]\bigr]
=\operatorname{Cov}_s\!\bigl[J_\theta(s)^{\top}(\varphi_\theta(s)-\mathbb{E}[\bar a\mid s])\bigr]$,
bounded by $G^{2}\operatorname{Var}_s\!\bigl(\varphi_\theta(s)-\mathbb{E}[\bar a\mid s]\bigr)$.
This term is also $O(\operatorname{diam}(\mathcal{A}')^{2})$ and hence
independent of $|\mathcal{A}|$, so the conclusion of
Theorem~\ref{thm:variance_app} carries over to the unconditional gradient noise.

\subsection{Grounding the Latent Lipschitz Assumption}
\label{app:latent_lipschitz}
Assumption~\ref{ass:latent_lipschitz} requires that proximity in the embedding space $\calZ$ implies similarity in the value space $Q$ in structured problems. This assumption is grounded in the smoothness of underlying physical or semantic attributes common in structured large action spaces:

\begin{itemize}[leftmargin=*,noitemsep,topsep=0pt]
\item \textbf{Logistics/Scheduling:} An action vector may represent, e.g., [resource, time, location]. Small perturbations in the latent space correspond to shifting a delivery by five minutes or 100 meters, which results in nearly identical costs ($Q$-values). This implies that the reward landscape is locally consistent with respect to the physical coordinates.
\item \textbf{Recommender Systems:} In item-factorization spaces, items with similar latent vectors (e.g., obtained using \textit{Word2Vec}) share similar user-preference profiles. The $Q$-value (predicted reward) is then expected to vary smoothly over the embedding space.
\end{itemize}

In practice, this assumption can be fulfilled by the structure inherent in the action space, e.g., in logistics or discretized robotics applications. Alternatively, learned or engineered embeddings $\phi$ can aim to approximate a \textit{bi-Lipschitz} mapping, where $L_1 \|\phi(a) - \phi(a')\| \le |Q(s, a) - Q(s, a')| \le L_2 \|\phi(a) - \phi(a')\|$. This bounded distortion ensures that the latent metric in $\calZ$ is a faithful proxy for functional similarity, justifying the distance-based regression in DBU. Because the $L_\infty$ metric used in SDN scales with the maximum coordinate-wise deviation, the search volume remains meaningful even when specific attributes (e.g., just 'time' or just 'location') are the primary drivers of $Q$-value variance.

Note that the Recommender environment used in this work is irregularly structured, explicitly not fulfilling this property.

\subsection{Continuity and Approximation Error}
\label{app:lipschitz_continuity}

Given the Latent Lipschitz Assumption, we can formulate that distance in the relaxed action space $\mathcal{A}'$ serves as a proxy for value difference in structured problems. We formalize this by assuming Lipschitz continuity of the action-value function over $\mathcal{A}'$.

\begin{proposition}[Approximation Bound via Lipschitz Continuity]
\label{prop:lipschitz_continuity_app}
Let $Q(s, \cdot): \mathcal{A}' \to \mathbb{R}$ be $L_Q$-Lipschitz continuous with respect to the Euclidean norm $\|\cdot\|_2$. Let $a^\star$ be the optimal discrete action, $\hat{a}$ a continuous proto-action, and $\bar{a}$ a target action. Let $a_{\mathrm{nn}}$ denote the nearest discrete neighbor of $\hat{a}$. Then
\(
Q(s, a^\star) - Q(s, a_{\mathrm{nn}})
\leq L_Q \left( \sqrt{J(\theta)} + \|\bar{a} - a^\star\|_2 + \varepsilon_{\mathrm{round}} \right),
\)
where $J(\theta) = \|\hat{a} - \bar{a}\|_2^2$ and $\varepsilon_{\mathrm{round}} = \|\hat{a} - a_{\mathrm{nn}}\|_2$.
\end{proposition}

\begin{proof}
We decompose the value difference by inserting the intermediate points $\bar a$ and $\hat a$:
\begin{align*}
Q(s,a^\star) - Q(s,a_{\mathrm{nn}})
&= \bigl[ Q(s,a^\star) - Q(s,\bar a) \bigr]
 + \bigl[ Q(s,\bar a) - Q(s,\hat a) \bigr] \\
&\quad + \bigl[ Q(s,\hat a) - Q(s,a_{\mathrm{nn}}) \bigr].
\end{align*}

By Lipschitz continuity of $Q$ w.r.t. $\|\cdot\|_2$,
\begin{align*}
Q(s,a^\star) - Q(s,a_{\mathrm{nn}})
&\le L_Q\, \|a^\star - \bar a\|_2
   + L_Q\, \|\bar a - \hat a\|_2
   + L_Q\, \|\hat a - a_{\mathrm{nn}}\|_2 \\
&= L_Q \left( \|a^\star - \bar a\|_2 + \|\bar a - \hat a\|_2 + \varepsilon_{\mathrm{round}} \right).
\end{align*}

Finally, since $J(\theta) = \|\hat a - \bar a\|_2^2$, we have
\[
\|\bar a - \hat a\|_2 = \sqrt{J(\theta)},
\]
which completes the proof.
\end{proof}

\subsection{Policy Improvement via DBU}
\label{app:improvement}

We now formalize how the Distance-Based Update (DBU) yields policy improvement in structured action spaces. We show that a DBU step increases $Q$ at the proto-action under local smoothness (Steps 1--2) and that this improvement transfers to the executed discrete action through the Lipschitz bound of Proposition~\ref{prop:lipschitz_continuity_app} (Step 3).

\begin{proposition}[Local Improvement under Smoothness]\label{prop:improvement_app}
Assume that, on a neighborhood of $\mu_{\mathrm{old}}:=\varphi_{\theta_{\mathrm{old}}}(s)$,
(i) $Q(s,\cdot)$ is $\beta$-smooth and twice continuously differentiable; (ii) the perturbation scale $\sigma_b$ used to construct
$\bar a$ is small enough that the first-order Taylor expansion of $Q$
in the perturbations is valid, and $\nabla_a Q(s,\mu_{\mathrm{old}})\neq 0$; (iii) the actor learning rate $\alpha$ is sufficiently small. Then a single DBU step on the actor satisfies
\[
Q\bigl(s,\varphi_{\theta_{\mathrm{new}}}(s)\bigr)
\;\ge\;
Q\bigl(s,\varphi_{\theta_{\mathrm{old}}}(s)\bigr)
\;+\;\Omega(\alpha\sigma_b^{2}/\tau_b)\cdot\bigl\|J_\theta(s)^{\top}\nabla_a Q(s,\mu_{\mathrm{old}})\bigr\|_{2}^{2}.
\label{eq:local-improve-cont}
\]
Combined with Lipschitz continuity of $Q$ on $\mathcal{A}'$
(Assumption~\ref{ass:latent_lipschitz}), this implies
\[
Q\bigl(s,a^{\star}\bigr) - Q\bigl(s,a_{\mathrm{exec}}\bigr)
\;\le\; L_Q\bigl(\sqrt{J(\theta_{\mathrm{new}})} + d(\bar a,a^{\star}) + \varepsilon_{\mathrm{round}}\bigr),
\label{eq:local-improve-disc}
\]
i.e., DBU improvement at the proto-action transfers to a bounded
sub-optimality gap for the executed discrete action $a_{\mathrm{exec}}$.
\end{proposition}

\begin{proof}
We separate the argument into three steps.

\paragraph{Step 1 -- Alignment of the DBU target with $\nabla_a Q$.}
The target $\bar a$ is constructed as the softmax-weighted average of $M$
perturbed candidates $a''_{(i)}=\mu_{\mathrm{old}}+\xi_{(i)}$ with
$\xi_{(i)}\sim\mathcal{N}(0,\sigma_b^{2}I)$ i.i.d.
In the population (continuum) limit and ignoring rounding (which we
re-introduce in Step 3),
\begin{equation}
\bar a - \mu_{\mathrm{old}}
\;=\;
\frac{\int \xi\,\exp\!\bigl(Q(s,\mu_{\mathrm{old}}+\xi)/\tau_b\bigr)\,p_{\sigma_b}(\xi)\,d\xi}
     {\int \exp\!\bigl(Q(s,\mu_{\mathrm{old}}+\xi)/\tau_b\bigr)\,p_{\sigma_b}(\xi)\,d\xi},
\label{eq:abar-integral}
\end{equation}
where $p_{\sigma_b}$ is the density of $\mathcal{N}(0,\sigma_b^{2}I)$.

By assumption (i)--(ii), $Q(s,\mu_{\mathrm{old}}+\xi)
=Q(s,\mu_{\mathrm{old}})+\nabla_a Q(s,\mu_{\mathrm{old}})^{\top}\xi+O(\|\xi\|^{2})$.
Substituting into \eqref{eq:abar-integral}, the leading term is the
expectation of $\xi$ under the tilted Gaussian
$\tilde p(\xi)\propto p_{\sigma_b}(\xi)\exp(\nabla_a Q^{\top}\xi/\tau_b)$,
which is itself Gaussian with mean $\sigma_b^{2}\nabla_a Q/\tau_b$.
Therefore
\begin{equation}
\bar a - \mu_{\mathrm{old}}
\;=\; \tfrac{\sigma_b^{2}}{\tau_b}\,\nabla_a Q(s,\mu_{\mathrm{old}})
\;+\; O(\sigma_b^{4}).
\label{eq:es-alignment}
\end{equation}
This is the evolution-strategies / NES gradient estimator
\citep[cf.][]{wierstra2014nes,salimans2017es}; it gives a constructive
proof that the DBU target points (to leading order in $\sigma_b$) along
the action-space Q-gradient. In particular,
$\nabla_a Q(s,\mu_{\mathrm{old}})^{\top}(\bar a - \mu_{\mathrm{old}})
=\sigma_b^{2}/\tau_b\,\|\nabla_a Q\|^{2}+O(\sigma_b^{4})\ge 0$.

\paragraph{Step 2 -- Local improvement at the proto-action.}
The DBU gradient flow on the actor is
$\theta_{\mathrm{new}} = \theta_{\mathrm{old}} + \alpha J_\theta(s)^{\top}(\bar a - \mu_{\mathrm{old}})$,
which to first order in $\alpha$ moves the proto-action by
\begin{equation}
\delta := \mu_{\mathrm{new}} - \mu_{\mathrm{old}}
= \alpha\,J_\theta(s)\,J_\theta(s)^{\top}(\bar a - \mu_{\mathrm{old}}) + O(\alpha^{2}).
\label{eq:delta}
\end{equation}
Since $J_\theta J_\theta^{\top}\succeq 0$, combining \eqref{eq:es-alignment}
and \eqref{eq:delta} yields
\begin{equation}
\nabla_a Q(s,\mu_{\mathrm{old}})^{\top}\delta
\;=\; \tfrac{\alpha\sigma_b^{2}}{\tau_b}\,
\bigl\|J_\theta(s)^{\top}\nabla_a Q(s,\mu_{\mathrm{old}})\bigr\|_{2}^{2}
+ O(\alpha\sigma_b^{4}+\alpha^{2}).
\label{eq:positive-inner}
\end{equation}

By assumption (i), $Q(s,\cdot)$ is $\beta$-smooth, so the descent lemma gives
\begin{equation}
Q(s,\mu_{\mathrm{new}}) \ge Q(s,\mu_{\mathrm{old}})
+ \nabla_a Q(s,\mu_{\mathrm{old}})^{\top}\delta - \tfrac{\beta}{2}\|\delta\|_{2}^{2}.
\label{eq:descent}
\end{equation}
The first-order term in \eqref{eq:descent} is $\Theta(\alpha)$ by
\eqref{eq:positive-inner}, while the quadratic term is $\Theta(\alpha^{2})$.
By assumption (iii), $\alpha$ is small enough that the first-order term
dominates, yielding \eqref{eq:local-improve-cont}.

\paragraph{Step 3 -- Bridge to discrete execution.}
Inequality \eqref{eq:local-improve-cont} bounds the change in the
continuous surrogate $Q(s,\varphi_\theta(s))$, not in the value of the
discretely executed policy. We close this gap using
Proposition~\ref{prop:lipschitz_continuity_app} (Approximation Bound via Lipschitz Continuity), which
applies to any state $s$, target $\bar a$, proto-action
$\hat a=\varphi_\theta(s)$, and nearest-neighbor or SDN-selected
discrete execution action $a_{\mathrm{exec}}$:
\begin{equation}
Q(s,a^{\star}) - Q(s,a_{\mathrm{exec}})
\;\le\; L_Q\Bigl(\sqrt{J(\theta)} + d(\bar a,a^{\star}) + \varepsilon_{\mathrm{round}}\Bigr),
\label{eq:lipschitz-bridge}
\end{equation}
where $J(\theta)=\|\varphi_\theta(s)-\bar a\|^{2}$ and
$\varepsilon_{\mathrm{round}}=\sup_{s,\theta}\|\varphi_\theta(s)-\mathrm{round}(\varphi_\theta(s))\|$.
Applied at $\theta=\theta_{\mathrm{new}}$, \eqref{eq:lipschitz-bridge}
yields \eqref{eq:local-improve-disc}: each of the three terms on the
right of \eqref{eq:lipschitz-bridge} is controlled -- $J(\theta_{\mathrm{new}})$
shrinks under the DBU regression objective, $d(\bar a,a^{\star})$ is
bounded by the perturbation radius and rounding accuracy, and
$\varepsilon_{\mathrm{round}}$ is at most half the lattice spacing.
\end{proof}

\paragraph{Note on the trust region.}
The smoothness step \eqref{eq:descent} uses an $L_{2}$ trust region in
continuous action space ($\|\delta\|$ is small because $\alpha$ is small),
not a KL trust region over the policy. Rounding $\varphi_\theta(s)$ to the integer
grid (App.~\ref{app:implementation}) enters only
through $\varepsilon_{\mathrm{round}}$ in \eqref{eq:lipschitz-bridge},
which is bounded by the lattice spacing and independent of $\theta$.

\paragraph{Note on stochasticity at training time.}
At inference, $a_{\mathrm{exec}}=\arg\max_{a'\in\mathcal{C}_{\mathrm{SDN}}}Q(s,a')$
satisfies \eqref{eq:lipschitz-bridge} directly. At training time the
SDN policy is stochastic (rank-based softmax over Q-values); taking
expectation of \eqref{eq:lipschitz-bridge} over the SDN sampling
distribution yields the same bound with an additional non-negative
exploration gap $\mathbb{E}_{a\sim\pi_{\mathrm{SDN}}}[Q(s,\arg\max)-Q(s,a)]$,
which is controlled by the exploration temperature $\tau_e$ (Sec.~\ref{sec:sdn}).

\newpage

\section{Hybrid Action Space Analysis}
\label{app:hybrid}

\subsection{Regret Analysis for Hierarchical Policy Decomposition}
\label{app:regret}

\begin{proposition}[Hierarchical Regret Floor]
\label{prop:regret_appendix}
Let a hierarchical policy be defined as $\pi(a|s) = \pi_{\mathrm{high}}(c|s)\,\pi_{\mathrm{low}}(a|s,c)$ over a partition $\{\mathcal{A}_c\}_{c \in \mathcal{C}}$. Let $c^\star$ denote the optimal cluster containing $a^\star$, and define $\epsilon_h = \sum_{c \neq c^\star} \pi_{\mathrm{high}}(c|s)$. Then the regret satisfies
\[
R_{H\text{-}DNC} \ge \epsilon_h \cdot 
\min_{c \neq c^\star}
\left( Q^\star(s, a^\star) - \max_{a \in \mathcal{A}_c} Q^\star(s, a) \right).
\]
\end{proposition}

\begin{proof}
Let $V^\pi(s) = \sum_{c} \pi_{\mathrm{high}}(c|s)\, V_c$, where $V_c = \mathbb{E}_{a \sim \pi_{\mathrm{low}}(\cdot|s,c)}[Q^\star(s,a)]$. Let $c^\star$ be the optimal cluster containing $a^\star$, so that $V^\star(s) = Q^\star(s,a^\star)$.

Then
\[
R = V^\star(s) - V^\pi(s)
= \sum_{c} \pi_{\mathrm{high}}(c|s)\left(Q^\star(s,a^\star) - V_c\right).
\]

For $c \neq c^\star$, we have
\[
V_c \le \max_{a \in \mathcal{A}_c} Q^\star(s,a),
\]
hence
\[
Q^\star(s,a^\star) - V_c \ge 
Q^\star(s,a^\star) - \max_{a \in \mathcal{A}_c} Q^\star(s,a).
\]

Therefore,
\[
R \ge \sum_{c \neq c^\star} \pi_{\mathrm{high}}(c|s)
\cdot 
\min_{c' \neq c^\star}
\left(Q^\star(s,a^\star) - \max_{a \in \mathcal{A}_{c'}} Q^\star(s,a)\right),
\]
which yields the stated bound.
\end{proof}

\subsection{Decoupled Optimization in Hybrid Spaces}
\label{app:decoupled}

\begin{lemma}[Decoupled Optimization in Hybrid Action Spaces]
\label{lem:hybrid_appendix}
Let $a=(a_d,a_c)$ with a separable metric $d^2 = d_d^2 + d_c^2$. If the coupling term in the Q-function is bounded by $\varepsilon_{\mathrm{couple}}$, then:
\begin{enumerate}[label=(\alph*)]
  \item If $\varepsilon_{\mathrm{couple}}=0$, DBU performs exact coordinate ascent on the components of $Q$.
  \item If $\varepsilon_{\mathrm{couple}}>0$, DBU performs approximate coordinate ascent with error $\mathcal{O}(\varepsilon_{\mathrm{couple}})$.
\end{enumerate}
\end{lemma}

\begin{proof} 
\textbf{(a) Exact separability:} The DBU loss decomposes as $\mathcal{L}(\theta) = d_d(\hat a_d,\bar a_d)^2 + d_c(\hat a_c,\bar a_c)^2$. Gradients w.r.t. outputs $\hat a_d$ and $\hat a_c$ are independent and passed to shared representation layers only subsequently. Since targets $\bar a_d, \bar a_c$ approximate component-wise maximizers via perturbation, sampling, and Q-based softmax, the update is equivalent to block coordinate ascent.
\textbf{(b) Bounded coupling:} The true joint improvement differs from the sum of component-wise improvements by the interaction term $R(s, a_d, a_c)$. Given $|R| \le \varepsilon_{\mathrm{couple}}$, the deviation from the optimal joint update is bounded by $2\varepsilon_{\mathrm{couple}}$, ensuring stability in hybrid spaces.
\end{proof}

\newpage

\section{Implementation details} \label{app:implementation}

\begin{figure}
    \centering
    \includegraphics[width=\linewidth]{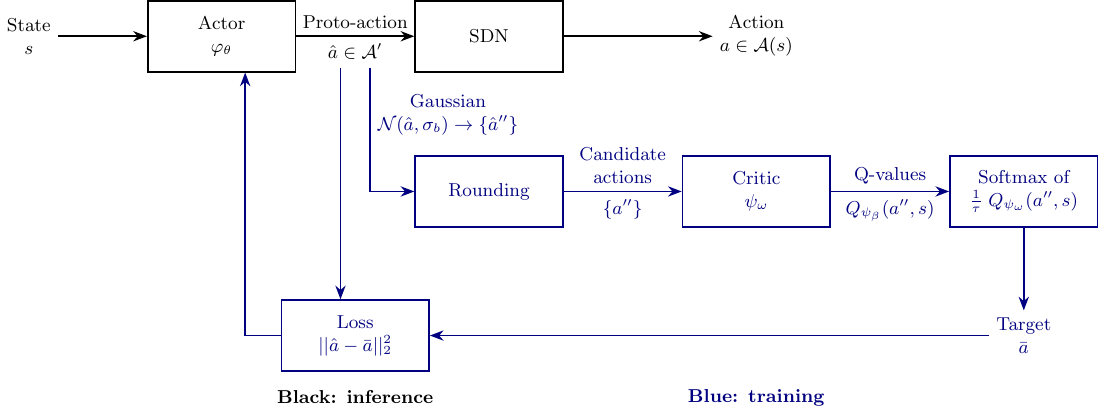}
    \caption{Schematic overview over full algorithm.}
    \label{fig:algorithm}
\end{figure}

We integrate DGRL into a standard off-policy actor-critic DRL paradigm: After initializing our neural networks, we run the algorithm for a fixed number of episodes. Within each episode, we collect experience by applying a policy $\pi(a|s)$ to the environment and storing the transitions in a replay buffer. In the Job Shop Scheduling, Joint Replenishment, and Recommender Environments, we apply a uniform random policy for the first 10\% of training episodes and start updating the networks after 5\% of training episodes. We update the neural networks the following way: every 8 steps, we sample a mini-batch of size 16 from the replay buffer. We update the critic using ordinary TD-errors with clipped double Q-learning and target networks. We update the target Q-networks after every critic update using polyak averaging with an update parameter of 0.02. We update the actor using the DBU update, storing all gradients for joint application after calculating the losses. We visualize the inference and training paradigm of DGRL in Figure~\ref{fig:algorithm}.

We employ feedforward neural networks with three hidden layers, the layer sizes are specified in App.~\ref{app:hyperparams}. The actor receives the state as input and outputs the proto-action as a vector. The critic receives the state and action as input and outputs a single Q-value. In the Maze and Recommender environments, the first layer of the critic only receives the state as an encoding layer. In the Job Shop Scheduling and Joint Replenishment environments, it receives both the state and the action immediately. We scale all state features to values between 0 and 1 before feeding them into the neural networks. In the Mazeworld and Recommender environments, we also apply a decoupled Fourier transformation with Fourier order 3 to improve the state representation, as in \citet{chandak2019} and \citet{akkerman2024dynamic}. The actor uses a tanh activation in its output layer. Since the proto-action is thus restricted to the range $(-1,1)$ as $(\hat{a}_\mathrm{min}, \hat{a}_\mathrm{max})$, we scale it to the action space $[A_\mathrm{min}, A_\mathrm{max}]$ with
\[
    \hat{a}_\mathrm{scaled} = \frac{\hat{a} - \hat{a}_\mathrm{min}}{\hat{a}_\mathrm{max} - \hat{a}_\mathrm{min}} (A_\mathrm{max} - A_\mathrm{min}) + A_\mathrm{min}.
\]
Continuous components of hybrid actions are scaled to the respective limits of the continuous action space accordingly. When given to the critic, we scale actions to the range $[0,1]$ to facilitate compatibility with the state representation using
\[
    a_\mathrm{critic} = \frac{a - A_\mathrm{min}}{A_\mathrm{max} - A_\mathrm{min}}.
\]

In SDN, we create options for the independent coordinate sampling by rounding the proto-action $\hat a$ up and down to the respectively next integer vector. We then increase or decrease the entries of the options vector until the maximum distance is reached, creating an options matrix. The time complexity of this operation is $\mathcal{O}(L)$. Once the matrix is created, we sample $2 \cdot K$ actions from it, treating each dimension independently. We then eliminate duplicates and -- in case of Euclidean or Manhattan Distances -- actions violating the distance constraints. Sampling the double amount of candidate actions ensures usually having sufficiently many candidate actions after these eliminations. Finally, we remove the surplus candidate actions and add the nearest neighbor of $\hat a$ to the set of candidate actions $\{a'\}$. We obtain the nearest neighbor by rounding $\hat a$ to the closest integer action in $\calA$. Adding the nearest neighbor ensures always having a valid action in the neighborhood in case of too restrictive distance constraints, as well as having access to the action most likely estimated by the actor for evaluation by the critic. We give the set of candidate actions $\{a'\}$ to the critic and select an action via sampling during training or choosing the action corresponding to the highest Q-value during testing. We ensure exploration during training in three ways: (i) we sample proto-actions from a multivariate Gaussian distribution $Z \sim \mathcal{N}(\varphi_\theta(s), \sigma_f)$, (ii) we sample candidate actions around $\hat a$, and (iii) we sample $a \in \{a'\}$ based on their Q-values. The first paradigm ensures exploration of different neighborhoods, while the second and third ensures exploration within neighborhoods. The stochastic sampling thereby allows for the efficient construction and exploration of larger neighborhoods than in Cacla or Wolpertinger.

In DBU, we perturb the proto-action $\hat a$ using a Gaussian distribution $Z \sim \mathcal{N}(\varphi_\theta(s), \sigma_b)$ and sample several candidate proto-actions $\hat a''$. We clip each $\hat a''$ to the closest integer action $a'' \in \calA$ using the same rounding function as in SDN. We then estimate the Q-value of each $a''$ and construct the target action $\bar a$ using a softmax over the Q-values. We finally update the actor by minimizing the distance $J(\theta)=\|\varphi_\theta(s) - \bar{a}\|_2^2$. As a safeguard against extreme distances, we implement the update using a Huber loss. Since $\hat a$ and $\bar a$ are constrained between -1 and 1, however, the Huber loss becomes linear only seldomly and the loss amounts to an MSE in most cases.

\newpage

\section{Experiments} \label{app:experiments}

In this section, we detail the experimental setup and used hardware in Section~\ref{sec:hardware}, explain the benchmark algorithms in Section~\ref{sec:baselines}, provide the results for the hyperparameter tuning in Section~\ref{app:hyperparams}, and provide descriptions of the environments in Section~\ref{sec:environments}

\subsection{Experimental Setup and Hardware} \label{sec:hardware}

We conducted our development and experiments on a high-performance cluster with 2.6Ghz CPUs with 56 threads and 64 RAM per node, as well as NVIDIA RTX 4090 GPUs. The algorithms are implemented in Python 3, we use PyTorch to construct neural network architectures \citep{pytorch}. The training times per algorithm are between 2 hours for VAC on the small Maze environment and 24 hours for DNC (SA) on the Joint Replenishment environment.

\subsection{Benchmark Algorithms}\label{sec:baselines}

\begin{table}
\centering
\caption{Comparison of architectural properties and scalability features.}
\begin{small}
\begin{tabularx}{\columnwidth}{l XXX}
\toprule
\textbf{Feature} & \textbf{Wolpertinger} & \textbf{DNC (SA)} & \textbf{DGRL (SDN + DBU)} \\
\midrule
\textbf{Search Strategy} & Local ($k$-NN) & Local (SA) & \textbf{Local (Sampling)} \\
\textbf{Metric Space} & $L_2$ Euclidean & Grid-based & \textbf{Any ($L_\infty$ preferred)} \\
\textbf{Gradient Logic} & DDPG & Policy Gradient & \textbf{Dist. Regression} \\
\textbf{Complexity} & $\mathcal{O}(|\mathcal{A}|)$ (preparation) & $\mathcal{O}(N \cdot Steps)$ & \textbf{$\mathcal{O}(N \cdot K)$ (parallel)} \\
\textbf{Sampling budget} & $k$ & $Steps \times (20 \times N + 1)$ & $K+1$ \textbf{(similar to $k$)} \\
\textbf{Hybrid Support} & No & No & \textbf{Joint Optimization} \\
\textbf{Robustness} & Low (Structure) & Low (Grid) & \textbf{High (Irregular)} \\
\bottomrule
\end{tabularx}
\end{small}
\label{tab:comparison}
\end{table}

We consider five benchmarks for the discrete environments: Cacla, Wolpertinger, DNC (SA), DNC(greedy), and LAR. For the hybrid action spaces, we consider three benchmarks: H-Cacla, H-DNC, and HyAR. Below, we describe each benchmark in more detail. For the complete explanations, we refer to the respective papers. We provide a comparative overview over different algorithmic paradigms in Table~\ref{tab:comparison}.

\paragraph{Vanilla Actor-Critic}~\\
We implement a standard actor-critic approach, referred to as Vanilla Actor-Critic (VAC), as a benchmark method. This benchmark utilizes a categorical policy $\pi(\boldsymbol{a} | \boldsymbol{s})$, representing the probability distribution over discrete actions $\boldsymbol{a}$ given state $\boldsymbol{s}$. The notation $\pi(\boldsymbol{a} | \boldsymbol{s})$ explicitly indicates the distributional nature of the policy. The actor directly outputs discrete actions, eliminating the need for the continuous proto-action generation.

VAC is trained using Advantage Actor-Critic (A2C), utilizing a Q-network as its critic. The critic is trained using on-policy TD-learning.

\paragraph{Cacla}~\\
Cacla was introduced by \citet{hasselt2009} as the first algorithm utilizing the continuous-discrete action mapping paradigm. The actor estimates a continuous proto-action $\hat a$, which is rounded to the closest integer action $a$. No further action search is conducted and the critic is not used for action selection.

Cacla is trained using Advantage Actor-Critic (A2C), utilizing a Q-network as its critic. The critic is trained using on-policy TD-learning.

In hybrid action space environments, we extend Cacla to create a continuous-to-discrete version of Parameterized Action Deep Deterministic Policy Gradients (PADDPG) \citep{Hausknecht2016}. In Hybrid Cacla (H-Cacla), the actor outputs a vector of continuous proto-actions and parameters at the same time. We then choose the discrete action component independently from the continuous parameter using Cacla. The independent action selection paradigm is common in the literature on hybrid action spaces and creates a popular and competitive benchmark \citep[cf.][]{Hausknecht2016, Fan2019}. As proposed by \citet{Hausknecht2016}, H-Cacla is trained using Deep Deterministic Policy Gradients (DDPG).

\paragraph{Wolpertinger}~\\
The Wolpertinger framework by \citet{dulac2015} employs approximate nearest neighbor search to identify discrete actions in $\mathcal{A}$ that are proximal to a continuous action proposal $\boldsymbol{\hat{a}}$. Specifically, the mapping function $h$ retrieves the $k$ closest discrete actions based on Euclidean distance:
\begin{equation*}
    h_k(\boldsymbol{\hat{a}}) = \argmin_{\boldsymbol{a} \in \mathcal{A}^k} \lVert \boldsymbol{a}-\boldsymbol{\hat{a}} \rVert_2.
\end{equation*}
Among these $k$ candidates, the action with the maximum $Q$-value is selected for execution in the environment. This selection strategy mirrors the approximate on-policy reasoning employed in {DNC}, with the key distinction that the critic evaluates candidates only \emph{after} neighborhood generation is complete. Hyperparameter values for $k$ across different experimental settings are documented in Section~\ref{app:hyperparams}.

Wolpertinger is trained using DDPG, utilizing a Q-network as its critic. The critic is trained using off-policy TD-learning.

\paragraph{DNC}~\\
{DNC}, as proposed in \citet{akkerman2024dynamic}, employs a structured search procedure to map continuous actions to discrete ones. First, a continuous action ${\hat{a}}$ from the actor is scaled and rounded to produce a discrete base action ${\bar{a}}$. DNC then generates a local neighborhood $\mathcal{A}'$ around ${\bar{a}}$ by systematically perturbing individual action dimensions using a predefined perturbation matrix. The critic evaluates all neighbors in $\mathcal{A}'$, and a simulated annealing-based search iteratively explores promising neighborhoods to escape local optima. At each iteration, the algorithm accepts neighbors with higher $Q$-values deterministically or accepts worse neighbors probabilistically according to a temperature-controlled acceptance criterion. The process continues until convergence criteria are met, returning the discrete action ${\bar{a}}^*$ with the highest $Q$-value encountered during the search. This approach scales efficiently to large action spaces by limiting neighborhood exploration to a radius of $(d\epsilon)$ around the base action while maintaining performance guarantees under local convexity assumptions.

DNC is trained using A2C, utilizing a Q-network as its critic. The critic is trained using on-policy TD-learning.

In hybrid action space environments, we extend DNC to create a continuous-to-discrete version of hierarchical action selection paradigms common in the literature on hybrid action spaces \citep[cf.][]{Xiong2018, Ma2021}. In such algorithms, an actor usually first estimates the continuous parameter independent of the discrete action. A critic or a second actor uses the state and this parameter as input and selects a discrete action dependent on the continuous parameter. Since we cannot only use the critic for action selection in large hybrid action spaces, we extend DNC to Hybrid DNC (H-DNC), relying on the actor and critic for action selection, as a challenging state-of-the-art benchmark for hybrid action spaces \citep[cf.][]{Ma2021}. In H-DNC, a first actor outputs $\boldsymbol{\hat{a}}_c$. A second actor receives the state and $\boldsymbol{\hat{a}}_c$ as input and estimates $\boldsymbol{\hat{a}}_d$. During DNC's search, the continuous component $\boldsymbol{\hat{a}}_c$ remains fixed and conditions the critic evaluations: $Q(\boldsymbol{s}, (\boldsymbol{a}'_d, \boldsymbol{\hat{a}}_c))$ for each discrete candidate $\boldsymbol{a}'_d$. The final action concatenates the DNC-selected discrete action with the original continuous output: $\boldsymbol{a} = (\boldsymbol{a}^*_d, \boldsymbol{\hat{a}}_c)$. Both actors of H-DNC are trained using A2C, in line with the use of policy gradient algorithms in the literature \citep[cf.][]{Ma2021}.

\paragraph{LAR}~\\
For LAR, as proposed in \citet{chandak2019}, we train a supervised model prior to the RL model to generate low-dimensional embeddings $\boldsymbol{e}'\in \mathbb{R}^{l}$ for each discrete action $\boldsymbol{a}$. This model is trained on a replay buffer containing state-action-next state tuples $(\boldsymbol{s}_t, \boldsymbol{a}_t, \boldsymbol{s}_{t+1})$, collected using a random exploration policy and capped at $6 \times 10^5$ transitions. 

The embedding model minimizes the KL-divergence between the true action distribution $\mathbb{P}(\boldsymbol{a}_t | \boldsymbol{s}_t,\boldsymbol{s}_{t+1})$ and its estimate $\hat{\mathbb{P}}(\boldsymbol{a}_t | \boldsymbol{s}_t,\boldsymbol{s}_{t+1})$, representing the likelihood of action $\boldsymbol{a}_t$ given the state transition. Training proceeds for up to 3000 epochs, though convergence typically occurs earlier. The embedding dimension $l$ is treated as a tunable hyperparameter and need not match the dimensionality of the discrete action space; specific values are reported in Section~\ref{app:hyperparams}.

During reinforcement learning, the continuous policy $\pi$ outputs an embedding vector $\boldsymbol{e}$. We retrieve the nearest learned embedding $\boldsymbol{e}'$ via $L_2$ distance and execute its corresponding discrete action $\boldsymbol{a}$. After each environment step, we perform one gradient update on the embedding model to refine the action representations $\boldsymbol{e}'$.

LAR is trained using A2C, utilizing a Q-network as its critic. The critic is trained using on-policy TD-learning.

\paragraph{HyAR}~\\
HyAR (Hybrid Action Representation) is proposed by \citet{Li2022} to deal with hybrid action spaces. It employs a conditional Variational Autoencoder (VAE) to learn compact latent representations for hybrid discrete-continuous action spaces. The framework maintains a learned embedding table for discrete actions and encodes continuous parameters through the VAE's latent space. The actor outputs a low-dimensional vector $[\boldsymbol{e}, \boldsymbol{z}_x] \in \mathbb{R}^{d_1 + d_2}$, where $\boldsymbol{e}$ is the discrete action embedding and $\boldsymbol{z}_x$ is the continuous parameter embedding. During decoding, $\boldsymbol{e}$ is mapped to a discrete action index $k$ via nearest-neighbor lookup in the embedding table, while $\boldsymbol{z}_x$ is decoded through the VAE conditioned on both state $\boldsymbol{s}$ and the selected discrete action $k$ to produce continuous parameters $\boldsymbol{x}_k$. The VAE is trained with three objectives: reconstruction of continuous parameters, KL-divergence regularization, and dynamics prediction to learn state transitions, ensuring the latent space captures both action semantics and environment dynamics.

HyAR is trained using A2C, utilizing a Q-network as its critic. The critic is trained using on-policy TD-learning.

\subsection{Hyperparameters}\label{app:hyperparams}

In below tables we report the found hyperparameter settings for each method-environment combination. In Tables~\ref{tab:hyperparameters_discrete1} and~\ref{tab:hyperparameters_discrete2} we report the hyperparameters for the discrete environments, and in Table~\ref{tab:hyperparameters_hybrid} for the hybrid environments. For all methods we tune the neural network sizes for both the actor and critic, and the actor and critic learning rates.

\textbf{Practical guide to DGRL hyperparameters.}
Some of DGRL's hyperparameters have little effect on performance: we always use $M=40$ (number of sampled candidates for DBU), $\tau_s=1$, $\tau_e=0.8$, and $\tau_b=0.01$ (temperatures). The neighborhood constraint $L$ and the number of sampled neighbors $K$ can be derived heuristically from the action space dimensionality: we use $L\approx[0.1 \cdot D]$ and $K \approx L \cdot N$. This heuristic ensures that, on average, every dimension of the action manifold is explored at least once within the specified neighborhood during a single update step. This provides a probabilistic lower bound on neighborhood coverage while maintaining the $O(N \cdot K)$ computational advantage that is central to DGRL. Fine-tuning these hyperparameters does not have a substantial effect on performance.

The exploration variance $\sigma_f$ and the perturbation variance $\sigma_b$ should usually have similar magnitudes: for them, we conduct a simple gridsearch between 0.01 and 1.0, ignoring combinations where one would be $>100\:\times$ higher than the other. Empirically, this choice is motivated by the observation that environments have certain local convexity properties of the value function, making the choice of similar standard deviations natural. In practice, this approach limits the tuning effort of DGRL. We tune the NN architectures per environment and the learning rates using ordinary RL gridsearches.

Using this method, all tested algorithms have similar hyperparameter grid sizes, as denoted in Table~\ref{tab:hyperparams_grid}. For the baselines, especially DNC, we follow the authors' guidelines for hyperparameter tuning, resulting in the displayed evaluated grid sizes.

\begin{table}[hbtp]
    \centering
    \def\arraystretch{1.1}
    \caption{Hyperparameters, set of values, and chosen values for the discrete environments. Abbreviations: S: structured, I: irregularly structured, lr: learning rate, Rcm: Recommender. In case of two entries, a linear decay is applied. Hyperparameters for DNC hold for DNC (SA), DNC (greedy) and Cacla.}
    \label{tab:hyperparameters_discrete1}
    \resizebox{\textwidth}{!}{\begin{tabular}{llccccc}
    \toprule
        & & & \multicolumn{4}{c}{Chosen values} \\
        \cmidrule(r){4-7}
        & Hyperparameters & Set of values & Maze, $5^4$ (S) & Maze, $5^4$ (I) & Maze, $17^{10}$ (S) & Maze, $17^{10}$ (I) \\
    \midrule
        \multirow{4}{*}{\begin{sideways} Overall \end{sideways}}
        & Actor nodes/layer & $\{32,64,128\}$ & $32$ & $32$ & $64$ & $64$ \\
        & Critic nodes/layer & $\{64,128,256\}$ & $64$ & $64$ & $128$ & $128$ \\
        & Actor lr $\alpha_\varphi$ & $\{10^{-4},10^{-5},10^{-6}\}$ & $5 \times 10^{-5}, 10^{-5}$ & $5 \times 10^{-5}, 10^{-5}$ & $10^{-5}, 5 \times 10^{-6}$ & $10^{-5}, 5 \times 10^{-6}$ \\
        & Critic lr $\alpha_\psi$ & $\{10^{-3},10^{-4},10^{-5}\}$ & $10^{-4}, 5 \times 10^{-5}$ & $10^{-4}, 5 \times 10^{-5}$ & $5 \times 10^{-5}, 10^{-5}$ & $5 \times 10^{-5}, 10^{-5}$ \\
    \hline
        \multirow{4}{*}{\begin{sideways} DGRL \end{sideways}}
        & Max distance $L$ & $\{1,2,4,10,20\}$ & $1$ & $1$ & $2$ & $2$ \\
        & Sampled actions $K$ & $\{10,20,40,100\}$ & $10$ & $10$ & $20$ & $20$ \\
        & Variance $\sigma_f$ & $\{0.01, 0.1,0.5,1.0\}$ & $0.5, 0.1$ & $0.5, 0.1$ & $0.5, 0.1$ & $0.5, 0.1$ \\
        & Perturbation var. $\sigma_b$ & $\{0.01,0.1,0.5,1.0\}$ & $0.5, 0.1$ & $0.5, 0.1$ & $0.5, 0.1$ & $0.5, 0.1$ \\
    \hline
        \multirow{4}{*}{\begin{sideways} DNC \end{sideways}}
        & SA search steps & $\{2,10,20,40\}$ & $2$ & $2$ & $2$ & $2$ \\
        & Cooling & $\{0.1,0.25\}$ & $0.25$ & $0.25$ & $0.25$ & $0.25$ \\
        & Accept. cooling & $\{0.1,0.25\}$ & $0.25$ & $0.25$ & $0.25$ & $0.25$ \\
        & Variance $\sigma_f$ & $\{0.01, 0.1,0.5,1.0\}$ & $1.0, 0.1$ & $1.0, 0.1$ & $0.5, 0.1$ & $0.5, 0.1$ \\
    \hline
        \multirow{3}{*}{\begin{sideways} LAR \end{sideways}}
        & Buffer size & $\{2 \times 10^4, 2 \times 10^5\}$ & $2 \times 10^5$ & $2 \times 10^5$ & n/a & n/a \\
        & Embedding lr $\alpha_{em}$ & $\{10^{-5}, 10^{-4}, 10^{-3}\}$ & $10^{-4}$ & $10^{-4}$ & n/a & n/a \\
        & Variance $\sigma_f$ & $\{0.01, 0.1,0.5,1.0\}$ & $0.5, 0.1$ & $0.5, 0.1$ & n/a & n/a \\
    \hline
        \multirow{2}{*}{\begin{sideways} {Wolp.} \end{sideways}}
        & $k$ & $\{10, 50, 100\}$ & $50$ & $50$ & n/a & n/a \\
        & Variance $\sigma_f$ & $\{0.01, 0.1,0.5,1.0\}$ & $0.5, 0.1$ & $0.1$ & n/a & n/a \\
    \bottomrule
    \end{tabular}}
\end{table}

\begin{table}[hbtp]
    \centering
    \def\arraystretch{1.1}
    \caption{Table \ref{tab:hyperparameters_discrete1} continued.}
    \label{tab:hyperparameters_discrete2}
    \resizebox{\textwidth}{!}{\begin{tabular}{llccccc}
    \toprule
        & & & \multicolumn{4}{c}{Chosen values} \\
        \cmidrule(r){4-7}
        & Hyperparameters & Set of values & Job Shop & Inventory & Rcm., $343^1$ & Rcm., $343^{10}$ \\
    \midrule
        \multirow{4}{*}{\begin{sideways} Overall \end{sideways}}
        & Actor nodes/layer & $\{32,64,128\}$ & $32$ & $32$ & $64$ & $128$ \\
        & Critic nodes/layer & $\{64,128,256\}$ & $64$ & $64$ & $128$ & $256$ \\
        & Actor lr $\alpha_\varphi$ & $\{10^{-4},10^{-5},10^{-6}\}$ & $5 \times 10^{-5}, 10^{-5}$ & $5 \times 10^{-5}, 10^{-5}$ & $5 \times 10^{-5}$ & $10^{-5}, 5 \times 10^{-6}$ \\
        & Critic lr $\alpha_\psi$ & $\{10^{-3},10^{-4},10^{-5}\}$ & $10^{-4}, 5 \times 10^{-5}$ & $10^{-4}, 5 \times 10^{-5}$ & $10^{-4}$ & $5 \times 10^{-5}, 10^{-5}$ \\
    \hline
        \multirow{4}{*}{\begin{sideways} DGRL \end{sideways}}
        & Max distance $L$ & $\{1,2,4,10,20\}$ & $1$ & $4$ & $10$ & $10$ \\
        & Sampled actions $K$ & $\{10,20,40,100\}$ & $10$ & $40$ & $10$ & $100$ \\
        & Variance $\sigma_f$ & $\{0.01, 0.1,0.5,1.0\}$ & $0.1, 0.01$ & $0.1, 0.05$ & $0.1$ & $0.1$ \\
        & Perturbation var. $\sigma_b$ & $\{0.01,0.1,0.5,1.0\}$ & $0.05, 0.01$ & $0.5, 0.2$ & $0.5, 0.1$ & $0.5, 0.1$ \\
    \hline
        \multirow{4}{*}{\begin{sideways} DNC \end{sideways}}
        & SA search steps & $\{2,10,20,40\}$ & $20$ & $40$ & $2$ & $10$ \\
        & Cooling & $\{0.1,0.25\}$ & $0.1$ & $0.1$ & $0.25$ & $0.25$ \\
        & Accept. cooling & $\{0.1,0.25\}$ & $0.1$ & $0.1$ & $0.25$ & $0.25$ \\
        & Variance $\sigma_f$ & $\{0.01, 0.1,0.5,1.0\}$ & $1.0, 0.1$ & $0.5, 0.1$ & $1.0, 0.1$ & $0.1$ \\
    \hline
        \multirow{3}{*}{\begin{sideways} LAR \end{sideways}}
        & Buffer size & $\{2 \times 10^4, 2 \times 10^5\}$ & n/a & n/a & $2 \times 10^5$ & n/a \\
        & Embedding lr $\alpha_{em}$ & $\{10^{-5}, 10^{-4}, 10^{-3}\}$ & n/a & n/a & $10^{-4}$ & n/a \\
        & Variance $\sigma_f$ & $\{0.01, 0.1,0.5,1.0\}$ & n/a & n/a & $1.0, 0.1$ & n/a \\
    \hline
        \multirow{2}{*}{\begin{sideways} {Wolp.} \end{sideways}}
        & $k$ & $\{10, 50, 100\}$ & n/a & n/a & $50$ & n/a \\
        & Variance $\sigma_f$ & $\{0.01, 0.1,0.5,1.0\}$ & n/a & n/a & $0.1$ & n/a \\
    \bottomrule
    \end{tabular}}
\end{table}

\begin{table}[hbtp]
	\centering
	\def\arraystretch{1.1}
	\caption{Hyperparameters, set of values, and chosen values for the hybrid environments. Abbreviations: S: structured, I: irregularly structured, lr: learning rate, Rcm: Recommender. In case of two entries, a linear decay is applied.}
	\label{tab:hyperparameters_hybrid}
    \resizebox{\textwidth}{!}{\begin{tabular}{llccccc}
    \toprule
        & & &  \multicolumn{4}{c}{Chosen values} \\
        \cmidrule(r){4-7}
        & Hyperparameters & Set of values & Maze, $5^4$ (S) & Maze, $17^{10}$, (S) & Rcm., $343^1$ & Rcm., $343^{10}$ \\
    \midrule
        \multirow{4}{*}{\begin{sideways} Overall \end{sideways}}
        & Actor nodes/layer & $\{32,64,128\}$ & $32$ & $64$ & $64$ & $128$ \\
        & Critic nodes/layer & $\{64,128,256\}$ & $64$ & $128$ & $128$ & $256$ \\
        & Actor lr $\alpha_\varphi$ & $\{10^{-4},10^{-5},10^{-6}\}$ & $5 \times 10^{-5}, 10^{-5}$ & $10^{-5}, 5 \times 10^{-6}$ & $5 \times 10^{-5}$ & $10^{-5}, 5 \times 10^{-6}$ \\
        & Critic lr $\alpha_\psi$ & $\{10^{-3},10^{-4},10^{-5}\}$ & $10^{-4}, 5 \times 10^{-5}$ & $5 \times 10^{-5}, 10^{-5}$ & $10^{-4}$ & $5 \times 10^{-5}, 10^{-5}$ \\
    \hline
        \multirow{4}{*}{\begin{sideways} DGRL \end{sideways}}
        & Max distance $L$ & $\{1,2,4,10,20\}$ & $1$ & $2$ & $10$ & $10$ \\
        & Sampled actions $K$ & $\{10,20,40,100\}$ & $10$& $20$ & $10$ & $100$ \\
        & Variance $\sigma_f$ & $\{0.01, 0.1,0.5,1.0\}$ & $0.5, 0.1$ & $0.5, 0.1$ & $0.1$ & $0.05$ \\
        & Perturbation var. $\sigma_b$ & $\{0.01,0.1,0.5,1.0\}$ & $0.5, 0.1$ & $0.5, 0.1$ & $0.1$ & $0.1$ \\
    \hline
        \multirow{4}{*}{\begin{sideways} H-DNC \end{sideways}}
        & SA search steps & $\{2,10,20,40\}$ & $2$ & $2$ & $2$ & $10$ \\
        & Cooling & $\{0.1,0.25\}$ & $0.25$ & $0.25$ & $0.25$ & $0.25$ \\
        & Accept. cooling & $\{0.1,0.25\}$ & $0.25$ & $0.25$ & $0.25$ & $0.25$ \\
        & Variance $\sigma_f$ & $\{0.01, 0.1,0.5,1.0\}$ & $0.5, 0.1$ & $1.0, 0.1$ & $1.0, 0.1$ & $0.1$ \\
    \hline
        \multirow{3}{*}{\begin{sideways} H-Cacla \end{sideways}}
        & \\
        & Variance $\sigma_f$ & $\{0.01, 0.1,0.5,1.0\}$ & $0.5, 0.1$ & $0.1$ & $1.0, 0.1$ & $0.5, 0.1$ \\
        & \\
    \hline
        \multirow{3}{*}{\begin{sideways} HyAR \end{sideways}}
        & Buffer size & $\{2 \times 10^4, 2 \times 10^5\}$ & $2 \times 10^5$ & n/a & $2 \times 10^5$ & n/a \\
        & VAE lr $\alpha_{em}$ & $\{10^{-5}, 10^{-4}, 10^{-3}\}$ & $10^{-3}$ & n/a & $10^{-4}$ & n/a \\
        & Variance $\sigma_f$ & $\{0.01, 0.1,0.5,1.0\}$ & $1.0, 0.1$ & n/a & $0.1$ & n/a \\
    \bottomrule
    \end{tabular}}
\end{table}

\begin{table}[th]
\centering
\caption{Hyperparameter tuning grid sizes.}
\begin{small}
\begin{tabularx}{\columnwidth}{l X c c}
\toprule
Algorithm & Hyperparameters tuned & Full grid & Evaluated \\
\midrule
DGRL & Standard deviations $\sigma_f$, $\sigma_b$ & 16 & 16 \\
DNC (SA/greedy) & SA search steps, cooling, acceptance cooling, $\sigma_f$ & 64 & 24 \\
Wolpertinger & $k$, $\sigma_f$ & 12 & 12 \\
LAR/HyAR & buffer size, VAR lr $\alpha_{em}$, $\sigma_f$ & 24 & 12 \\
\bottomrule
\end{tabularx}
\end{small}
\label{tab:hyperparams_grid}
\end{table}

\subsection{Environments}\label{sec:environments}

In the following, we provide a description of the four environments and twelve specifications used. For each environment, we explain the environment dynamics, the action space, and the variants used in this work.

\paragraph{Mazeworld}~\\
Mazeworlds are a classic family of environments used in many works on scalable DRL \citep[e.g.,][]{chandak2019, akkerman2024dynamic}. Mazeworlds allow the exact specification of action spaces and transition dynamics, making them a natural choice for algorithmic evaluation.

An agent navigates a 2D rectangular space with continuous states. The space has walls that cannot be crossed and a target area that should be reached. Every step incurs a negative reward of $-0.05$, reaching the target area incurs a positive reward of $+10$. At each step, the agent can choose $d$ out of $n$ actuators. The first actuator encodes the ``do nothing" action, the other actuators encode evenly spaces movement vector of equal length. An action is generated by the linear combination of the vectors chosen by the agent. A maximum step size limits the length of the resulting movement vector, the borders of the space and the walls further restrict agent movements. With 10\% probability, stochastic noise is added to the movement vector. An episode has at most 100 steps and terminates when the target region is reached. Figure~\ref{fig:maze} illustrates the maze environment.

In the structured versions, the actuators follow a natural sequence - i.e., 90\textdegree\, is followed by 180\textdegree\, and so on. In the irregularly structured versions, this sequence except for the ``do nothing" action is shuffled, creating a random sequence of actuators.
The action space of the discrete versions thus has the size $D^N$ and the shape of a vector of length $N$ with each entry ranging between $0$ and $D$.
In the hybrid versions, the movement direction is estimated in the same way as in the discrete versions using a linear combination of actuators. In addition, the actor estimates a continuous parameter. After being clipped to the maximum step length, this parameter determines the step size.
The action space of the hybrid versions thus has the shape of a vector of length $N+1$ with the first $N$ entries ranging between $0$ and $D$ and the last entry ranging between $0$ and the maximum step size.

\begin{figure}[hbtp]
    \centering
    \includegraphics[width=0.23\textwidth]{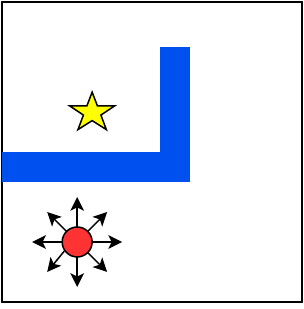}
    \caption{Illustration of the maze environment. Red dot denotes agent, yellow star denotes target, blue bars  denotewalls, black arrows denote actuators.}
    \label{fig:maze}
\end{figure}

\paragraph{Job Shop Scheduling}~\\
The Job Shop Scheduling problem we use was proposed by \citet{akkerman2024dynamic} as a structured problem with a large discrete structured action space. It is based on real-world job-shop scheduling problems common in industrial decision-making.

Each machine $i$ has different energy consumption per job. The machine deterioration factor $w_i \in [0,10]$ increases when utilization exceeds $75\%$ of capacity $D$ (wear from over-use) and decreases when utilization falls below $50\%$ (machine repair). Deterioration and repair are stochastic: deterioration $\Delta w_i \sim \text{Uniform}[0.05, 0.4]$, repair $\Delta w_i \sim \text{Uniform}[-0.4, -0.05]$. The factor $w_i$ corresponds to the energy cost per job, capped at $M_c$. Initially $w_i=1.0$ for all machines. The reward function is:
\begin{equation}
    R_t = \sum_{i=1}^N \left(a_ip - \min\{a_i \cdot (1+w_i), M_c\}\right) - \sigma(a),
\end{equation}
where $a_i$ denotes jobs allocated to machine $i$, $p=2$ is the reward per finished job, $M_c=100$ caps energy consumption, and $\sigma(a)$ penalizes load imbalance via standard deviation. We set $D=10$, results use $N=10$ machines.

\paragraph{Joint Replenishment (Inventory)}~\\
The Joint Replenishment problem is based on real-world inventory management problems, e.g., spare parts for down-time critical assets \citep{zhu_2020} or semiconductors \citep[][] {fleuren_2025} and in retail supply chains \citep{lowery_2026}. The problem we study was proposed by \citet{vanvuchelen2022} and used by \citet{akkerman2024dynamic}.

We set per-item costs as: holding $h_i=1$, backorder $b_i=19$, ordering $o_i=10$, and joint ordering $O=75$. Order-up-to levels range from $[0,66]$, and demand follows a Poisson distribution with $\lambda_i \in \{10, 20\}$ (half of items each). Initial inventory is $25$ per item, and episodes run for $100$ timesteps. The reward function is:
\begin{equation}
   R_t = - \sum_{i=1}^N \left( h_{i}I_{i,t}^+ + b_iI_{i,t}^- + o_i a_i \right) + O\mathds{1}_{\{\sum_{i=1}^N{q^o_{i,t}}>0\}},
\end{equation}
where $I_{i,t}^+$ and $I_{i,t}^-$ denote positive and negative inventory levels, and $\mathds{1}_{\{\sum_{i=1}^N{q^o_{i,t}}>0\}}$ indicates if any item is ordered (if order quantity $q^o_{i,t} > 0$).

\paragraph{Recommender}~\\
Recommender systems are classical environments with large irregularly structured action spaces \citep[e.g.,][]{dulac2015, chandak2019}. We utilize the MovieLens 25M dataset, which contains metadata and user preferences on movies \citep{Harper2015}. Based on this data, we construct a simulation of user behavior when interacting with a movie platform. We construct a feature vector per movie by vectorizing the list of movies based on their genre description using a combined \textit{term-frequency} (tf) and \textit{inverse-document-frequency} (idf) vectorizer \citep[cf.][]{Pedregosa2011}. As the resulting matrix contains several duplicates, i.e., movies with the exact same combination of features, we retrieve only unique feature vectors. We base the conditional probability of a customer picking movie $j$ if the last movie picked was $i$ on the cosine similarity of both movies' feature vectors. Cosine similarity \(S_{ij}\) between two movies \(i\) and \(j\) is computed as
\begin{equation}
S_{ij} = \frac{\boldsymbol{T}^{\mathrm{tf-idf}}_i \cdot \boldsymbol{T}^{\mathrm{tf-idf}}_j}{||\boldsymbol{T}^{\mathrm{tf-idf}}_i||\,||\boldsymbol{T}^{\mathrm{tf-idf}}_j||}\enspace.
\end{equation}
We then obtain a probability $\Tilde{P}_{ij}$ of picking recommended movie $j$ -- when the last picked movie was $i$ -- by applying a sigmoid function to each $S_{ij}$, yielding
\begin{equation}
\Tilde{P}_{ij} = \frac{1}{1+\exp{(-5\cdot S_{ij})}}\enspace. 
\end{equation}

In the discrete versions, based on similarity to the previously watched movie, the user can select one of the recommended movies or a different movie. If the user selects one of the recommended movies, the probability of the user leaving the system after watching the movie is $10\%$. In case of selecting a non-recommended movie, the probability is $20\%$. This corresponds to the setting studied in \citet{dulac2015} and simulates user patience. The agent then collects a movie-specific reward derived from the dataset. An episode has at most $100$ steps and terminates when the customer leaves the site.

In the hybrid versions, a price for each recommended movie is estimated. The user chooses one of the recommended movies or another movie based on a Multinomial logit (MNL) customer choice model with movie similarities and prices as input. The MNL model works as follows: the utility for each recommended movie $j$ is given by
\begin{equation}
U_j = \kappa_{\text{similarity}} \cdot S_{ij} - \kappa_{\text{price}} \cdot p_j + \varepsilon_j,
\end{equation}
where $\kappa_{\text{similarity}}$ and $\kappa_{\text{price}}$ are preference parameters for similarity and price sensitivity respectively, $S_{ij}$ is the cosine similarity between the  last watched movie $i$ and movie $j$, $p_j$ is the price of movie $j$, and $\varepsilon_j$ is an i.i.d. standard Gumbel noise term, as commonly used for MNL models \citep[cf.][]{book_train_2009}. An outside option (selecting a non-recommended movie) has utility
\begin{equation}
U_0 = \kappa_{\text{outside}} + \varepsilon_0,
\end{equation}
where $\kappa_{\text{outside}}$ captures the benchmark utility of browsing and $\varepsilon_0$ is also an i.i.d. standard Gumbel error term. The user selects the option with the highest utility. The parameters $\kappa_{\text{similarity}}=2$, $\kappa_{\text{price}}=3$, and $\kappa_{\text{outside}}=0$ are tuned to ensure that users exhibit price-sensitive behavior while avoiding overly prescriptive effects of the prices. If a recommended movie is chosen, the reward is the movie-specific reward plus the price. If a non-recommended movie is chosen, the reward is only the movie-specific reward. Termination probabilities and episode lengths are analogous to the discrete case.

The action space of the discrete versions thus has a size of $D^N$ and the shape of a vector of length $N$ with each entry ranging between $0$ and $D$. The action space of the hybrid versions has the shape of a vector of length $2 \cdot N$ with the first $N$ entries ranging between $0$ and $D$ and the last $N$ entries ranging between $0$ and the maximum price.

The environment is initialized by a user selecting a random movie. At each step, the agent recommends $N$ out of $D$ movies to the user.

\newpage

\section{Auxiliary Results} \label{app:results}

In this section, we present complementary results. First, we conduct an analysis of different distance metrics used in DBU in Section~\ref{sec:dist_metr}, next, we conduct an ablation study for DGRL in Section~\ref{sec:ablation}, and we end with numerical results for all tested algorithms in Section~\ref{sec:table_results}, as complement to the results in the main text.

\subsection{Results for distance metrics}\label{sec:dist_metr}

\begin{figure}[hbtp]
    \centering
    \begin{minipage}{0.4\linewidth}
        \centering
        \includegraphics[width=\linewidth]{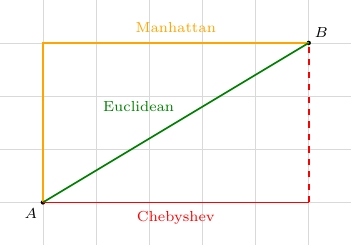}
    \end{minipage}
    \begin{minipage}{0.38\linewidth}
        \includegraphics[width=\linewidth]{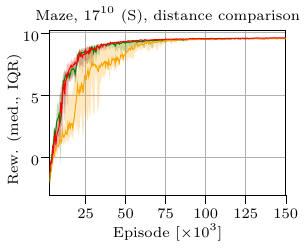}
    \end{minipage} \\
    \includegraphics[width=0.85\textwidth]{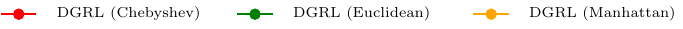}
    \caption{Left: schematic representation of Chebyshev, Euclidean, and Manhattan Distance. Right: Comparison of different distance metrics for DBU, median performance over 10 random training seeds.}
    \label{fig:chebyshev}
\end{figure}

We compare DGRL using a Chebyshev Distance, a Euclidean Distance, and a Manhattan Distance as neighborhood constraint on the large structured Maze environment. We display a schematic representation of the different distance metrics and the results of out tests in Figure~\ref{fig:chebyshev}. For the experiments, we re-tuned the distance parameters of DGRL using the alternative distance metrics to allow for a fair comparison. For the Euclidean Distance, we chose $L=6$ and $K=60$. For the Manhattan Distance, we chose $L=8$ and $K=80$, while we used $L=2$ and $K=20$ for the Chebyshev distance.

We observe similar performance across distance metrics, with all variations of DGRL using all random seeds converging to the same policy. The convergence using the Chebyshev distance is most stable, while DGRL using a Euclidean Distance has some instability in the beginning and DGRL using a Manhattan Distance performing worse than the other variants at the start. We can find the empirical reason for this observation in the neighborhood construction: while the Chebyshev Distance constraints every dimension $N \in \{1,\dots,N\}$ independently, the Euclidean and Manhattan Distances constrain the $N$-dimensional vector. To allow for a reasonable number of neighbors to be evaluated, we therefore need a larger neighborhood constraint $L$, accounting for the larger distances between $\hat a$ and $a'$. This leads to cases in which the distance in one dimension is substantially larger than in the other dimensions of the action vector, extending the trust region around the proto-action. This can lead to instability during training and delay convergence.

In addition, the computational efficiency of SDN is decreased when using a Euclidean or Manhattan Distance. Since SDN scales with $\mathcal{O}(N \cdot K)$, requiring a larger $K$ leads to higher computational costs, as does the additionally required check of the maximum distance to the proto-action during sampling.

\subsection{Results of ablation study}\label{sec:ablation}

\begin{figure}[th]
    \includegraphics[width=0.345\textwidth]{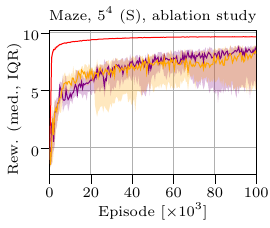}%
    \hfill%
    \includegraphics[width=0.315\textwidth]{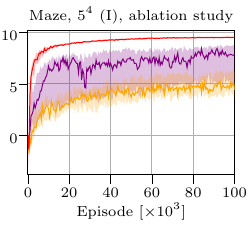}%
    \hfill%
    \includegraphics[width=0.325\textwidth]{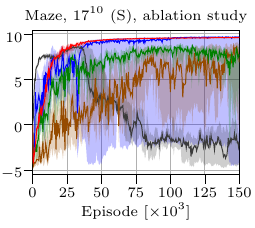}%
    \hfill%
    \par\smallskip
    \includegraphics[width=\textwidth]{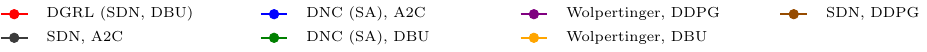}
    \caption{Results for ablation studies in small structured, small irregularly structured, and large structured Maze environment. All results report median rewards over ten random seeds.}
    \label{fig:results_ablation}
\end{figure}

Since DGRL encompasses the two novel components SDN and DBU, we perform an ablation study to differentiate the contribution of each component. We report the results of this ablation study in Figure~\ref{fig:results_ablation}. We investigate the contributions of SDN and DBU by testing Wolpertinger with DDPG and DBU, DNC with A2C and DBU, and SDN with DBU, A2C, and DDPG in the Maze environment. We use the large structured Mazeworld for these experiments, except for Wolpertinger, where scalability problems prevent its application. We therefore test the versions of Wolpertinger on the small structured and irregularly structured Maze environments.

We observe both SDN and DBU contributing to the stability and performance of DGRL: when comparing different loss functions for SDN, DDPG displays unstable training and converges to a lower final performance compared to DBU, while A2C displays high variance and very low final performance, the training essentially crashes. We can attribute the decreased stability of DDPG to its local gradient ascent nature -- while DBU samples candidate actions globally, DDPG searches for the steepest local gradient ascent to update the actor, thus being more likely to reach local optima. In addition, DDPG requires the critic to output meaningful Q-values on the relaxed action space $\calA'$, while DBU maps all actions to $\calA$ before giving them to the critic. Since the critic is trained only on feasible actions $a \in \calA$, thus not needing to account for the neighborhood mapping, this challenges the critic. For A2C, the variance of policy gradient algorithms in large action spaces hinders convergence, especially given the mismatch between the parameterized probabilities of proto-actions $\hat a$ depending on states and the probabilities of actions $a$ depending on states and mapping functions. Furthermore, the on-policy algorithm cannot recover after bad updates and thus converges to a low performance.

Considering mapping functions, we observe both Wolpertinger and DNC using DBU performing worse than DGRL. In case of Wolpertinger, the performance difference is the same as when using DDPG in the structured environment, but increases in the irregularly structured environment. In case of DNC, the performance gap similarly increases when using DBU. This highlights the unique link between SDN and DBU: since SDN employs coordinate-independent stochastic sampling in the neighborhood of $\hat a$, an algorithm minimizing per-dimension distances to high-value actions is optimal to update the actor. In contrast, Wolpertinger explores the direct neighborhood of $\hat a$ exhaustively, but lacks exploration of its larger surrounding area in the action space. Additionally, it learns a deterministic policy, which can lead to disadvantages given the stochastic transition kernel of the Maze and other environments. DNC explores the neighborhood in a grid-like fashion. As seen in Section~\ref{sec:exp_setup} and here, this search along grid lines performs well in structured environments, but fails in absence of structure. Furthermore, the Hamming-constrained sequential grid search of DNC can lead individual dimensions of the action vector far away from $\hat a$, which the Chebyshev constraints of SDN generally avoid, thereby stabilizing convergence and performance.

\subsection{Generalization results}

We test how trained models generalize beyond the immediate setting they were trained on in the example of the Maze environment (Table~\ref{tab:results_generalization}). In the unmodified structured and irregularly structured $5^4$ Mazes, DGRL and DNC (SA) perform equally well, while Cacla never reaches the goal. In the structured Maze, DGRL retains most of its performance when two actuators are added ($D+2$), whereas DNC (SA) fails; for $D+4$ and beyond, all methods fail. In the irregularly structured Maze, no method generalizes to additional actuators. Adding actor outputs is infeasible due to the design of the neural networks. Shuffling the actor outputs leaves all methods unaffected, but shuffling the actuators makes all methods fail. These results are consistent with DGRL relying on a fixed representation of the action space (Sec.~\ref{sec:methodology}): it tolerates small, structure-preserving alterations of the action space, but not changes of the mapping between proto-actions and actions.

\begin{table}[th]
\centering
\caption{Generalization of trained Maze models to modified action spaces without retraining. Median test reward over 10 seeds (100 episodes each). Reward $-4.95$ means the goal is not reached within 100 steps. Dashes denote settings where the checkpoint cannot be loaded.}
\begin{small}
\setlength{\tabcolsep}{3pt}
\begin{tabularx}{\columnwidth}{l *{6}{>{\centering\arraybackslash}X}}
\toprule
& \multicolumn{3}{c}{Maze, $5^4$, structured} & \multicolumn{3}{c}{Maze, $5^4$, irregularly structured} \\
\cmidrule(lr){2-4} \cmidrule(lr){5-7}
Modification & DGRL & DNC (SA) & Cacla & DGRL & DNC (SA) & Cacla \\
\midrule
None (baseline) & 9.58 & 9.57 & -4.95 & 9.50 & 9.50 & -4.95 \\
More actuators ($D+2$) & 8.28 & -4.45 & -4.95 & -4.95 & -4.95 & -4.95 \\
More actuators ($D+4$) & -4.48 & -4.95 & -4.95 & -4.95 & -4.95 & -4.95 \\
More actuators ($D+8$) & -4.95 & -4.95 & -4.95 & -4.95 & -4.95 & -4.95 \\
More actuators ($D+16$) & -4.95 & -4.95 & -4.95 & -4.95 & -4.95 & -4.95 \\
Larger $N$ (actor outputs) & - & - & - & - & - & - \\
Shuffled actor outputs & 9.58 & 9.57 & -4.95 & 9.50 & 9.50 & -4.95 \\
Shuffled actuators & -4.95 & -4.95 & -4.95 & -4.95 & -4.95 & -4.95 \\
\bottomrule
\end{tabularx}
\end{small}
\label{tab:results_generalization}
\end{table}

\subsection{Numerical results}\label{sec:table_results}

We provide numerical results for all tested algorithms in all discrete environments in Table~\ref{tab:disc_res1} and Table~\ref{tab:disc_res2}, and for all hybrid environments in Table~\ref{tab:hybr_res}. We display the peak mean and median performance over ten random seeds, as well as the standard deviation of the peak performances. Furthermore, we display the percentage differences of medians against DNC (H-DNC) and Cacla (H-Cacla) for the discrete (hybrid) environments.

\begin{table}
\centering
\caption{Numerical results for discrete environments. Statistical metrics are calculated over 10 random seeds. Means and medians are maxima over the entire training time. Percentage differences compared to DNC and Cacla are calculated using medians.}
\footnotesize
\begin{tabularx}{\textwidth}{Xrrrrrr}
\toprule
\midrule
Algorithm & Mean & Median & Std. & vs. DNC (\%) & vs. Cacla (\%) & Num Seeds \\
\midrule
\midrule
\multicolumn{7}{c}{Maze, $5^4$, structured, discrete} \\
\midrule
{DGRL (SDN, DBU)} & 9.67 & 9.67 & 0.01 & 0 & 14390 & 10 \\
{DNC (SA), A2C} & 9.74 & 9.75 & 0.03 & 0 & 14507 & 10 \\
{DNC (greedy), A2C} & 9.35 & 9.36 & 0.17 & -3 & 13927 & 10 \\
{Wolpertinger, DDPG} & 8.24 & 8.91 & 1.29 & -8 & 13248 & 10 \\
{Wolpertinger, DBU} & 8.31 & 8.63 & 1.13 & -11 & 12822 & 10 \\
{Cacla, A2C} & 1.36 & 0.07 & 2.31 & -99 & 0 & 10 \\
{VAC, A2C} & 0.99 & 0.98 & 0.64 & -89 & 1374 & 10 \\
{LAR, A2C} & -0.01 & -0.02 & 0.03 & -100 & -136 & 10 \\
\midrule
\midrule
\multicolumn{7}{c}{Maze, $5^4$, irregularly structured, discrete} \\
\midrule
{DGRL (SDN, DBU)} & 9.57 & 9.56 & 0.07 & -1 & 1569 & 10 \\
{DNC (SA), A2C} & 8.99 & 9.70 & 1.78 & 0 & 1594 & 10 \\
{DNC (greedy), A2C} & 8.58 & 9.27 & 1.67 & -4 & 1518 & 10 \\
{Wolpertinger, DDPG} & 7.27 & 8.69 & 2.23 & -10 & 1416 & 10 \\
{Wolpertinger, DBU} & 5.96 & 5.54 & 1.15 & -42 & 867 & 10 \\
{Cacla, A2C} & 0.96 & 0.57 & 1.22 & -94 & 0 & 10 \\
{VAC, A2C} & 1.62 & 1.66 & 0.55 & -82 & 190 & 10 \\
{LAR, A2C} & 0.00 & -0.02 & 0.04 & -100 & -104 & 10 \\
\midrule
\midrule
\multicolumn{7}{c}{Maze, $17^{10}$, structured, discrete} \\
\midrule
{DGRL (SDN, DBU)} & 9.61 & 9.63 & 0.04 & 0 & 19548 & 10 \\
{DNC (SA), A2C} & 8.79 & 9.69 & 1.33 & 0 & 19668 & 10 \\
{DNC (greedy), A2C} & 6.75 & 7.28 & 2.72 & -24 & 14804 & 10 \\
{Cacla, A2C} & 1.15 & -0.05 & 2.01 & -100 & 0 & 10 \\
{SDN, A2C} & 8.65 & 8.73 & 0.42 & -9 & 17740 & 10 \\
{SDN, DDPG} & 7.88 & 8.79 & 2.33 & -9 & 17863 & 10 \\
{DNC (SA), DBU} & 9.16 & 9.22 & 0.16 & -4 & 18730 & 10 \\
{SDN (EU), DBU} & 9.57 & 9.60 & 0.08 & 0 & 19494 & 10 \\
{SDN (MH), DBU} & 9.56 & 9.61 & 0.14 & 0 & 19505 & 10 \\
\midrule
\midrule
\multicolumn{7}{c}{Maze, $17^{10}$, irregularly structured, discrete} \\
\midrule
{DGRL (SDN, DBU)} & 9.25 & 9.29 & 0.18 & 66 & 2290 & 10 \\
{DNC (SA), A2C} & 5.50 & 5.59 & 3.28 & 0 & 1338 & 10 \\
{DNC (greedy), A2C} & 5.26 & 4.86 & 3.27 & -13 & 1150 & 10 \\
{Cacla, A2C} & 1.07 & 0.39 & 1.47 & -93 & 0 & 10 \\
\midrule
\midrule
\multicolumn{7}{c}{Job Shop Scheduling, $10^{10}$} \\
\midrule
{DGRL (SDN, DBU)} & 4027.51 & 4033.32 & 227.08 & 692782 & 494 & 10 \\
{DNC (SA), A2C} & -661.19 & 0.58 & 1985.72 & 0 & -99 & 10 \\
{DNC (greedy), A2C} & -761.06 & 1711.94 & 4839.12 & 293993 & 152 & 10 \\
{Cacla, A2C} & 492.07 & 678.30 & 3387.28 & 116424 & 0 & 10 \\
\midrule
\midrule
\multicolumn{7}{c}{Joint Replenishment, $67^{20}$} \\
\midrule
{DGRL (SDN, DBU)} & -335780.23 & -328335.73 & 22597.56 & 9 & 6 & 10 \\
{DNC (SA), A2C} & -360626.61 & -362375.81 & 11456.72 & 0 & -3 & 10 \\
{DNC (greedy), A2C} & -355996.16 & -349061.74 & 19231.48 & 3 & 0 & 10 \\
{Cacla, A2C} & -350049.75 & -350824.13 & 6771.59 & 3 & 0 & 10 \\
\midrule
\bottomrule
\end{tabularx}
\label{tab:disc_res1}
\end{table}

\begin{table}
\centering
\caption{Table~\ref{tab:disc_res1} continued.}
\footnotesize
\begin{tabularx}{\textwidth}{Xrrrrrr}
\toprule
\midrule
Algorithm & Mean & Median & Std. & vs. DNC (\%) & vs. Cacla (\%) & Num Seeds \\
\midrule
\multicolumn{7}{c}{Recommender, $343^1$, discrete} \\
\midrule
{DGRL (SDN, DBU)} & 2009.88 & 2006.62 & 42.18 & 31 & 60 & 10 \\
{DNC (SA), A2C} & 1534.94 & 1524.49 & 265.15 & 0 & 21 & 10 \\
{DNC (greedy), A2C} & 1587.59 & 1637.68 & 273.37 & 7 & 30 & 10 \\
{Wolpertinger, DDPG} & 460.38 & 444.65 & 29.42 & -70 & -64 & 10 \\
{Cacla, A2C} & 1370.89 & 1253.92 & 320.12 & -17 & 0 & 10 \\
{VAC, A2C} & 826.41 & 876.02 & 152.60 & -42 & -30 & 10 \\
{LAR, A2C} & 830.31 & 843.62 & 231.14 & -44 & -32 & 10 \\
\midrule
\midrule
\multicolumn{7}{c}{Recommender, $343^{10}$, discrete} \\
\midrule
{DGRL (SDN, DBU)} & 1937.61 & 1941.89 & 53.86 & 147 & 164 & 10 \\
{DNC (SA), A2C} & 802.80 & 783.96 & 49.54 & 0 & 6 & 10 \\
{DNC (greedy), A2C} & 730.46 & 730.52 & 15.86 & -6 & 0 & 10 \\
{Cacla, A2C} & 730.68 & 732.85 & 15.27 & -6 & 0 & 10 \\
\midrule
\bottomrule
\end{tabularx}
\label{tab:disc_res2}
\end{table}

\begin{table}
\centering
\caption{Numerical results for hybrid environments. Statistical metrics are calculated over 10 random seeds. Means and medians are maxima over the entire training time. Percentage differences compared to DNC and Cacla are calculated using medians.}
\footnotesize
\begin{tabularx}{\textwidth}{Xrrrrrr}
\toprule
\midrule
Algorithm & Mean & Median & Std. & vs. H-DNC (\%) & vs. H-Cacla (\%) & Num Seeds \\
\midrule
\midrule
\multicolumn{7}{c}{Maze, $5^4$, hybrid} \\
\midrule
{DGRL (SDN, DBU)} & 9.71 & 9.71 & 0.01 & 0 & 0 & 10 \\
{Hybrid DNC, H-A2C} & 9.78 & 9.79 & 0.02 & 0 & 1 & 10 \\
{Hybrid Cacla, DDPG} & 9.57 & 9.66 & 0.17 & -1 & 0 & 10 \\
{HyAR, A2C} & 0.27 & 0.25 & 0.23 & -97 & -97 & 10 \\
\midrule
\midrule
\multicolumn{7}{c}{Maze, $17^{10}$, hybrid} \\
\midrule
{DGRL (SDN, DBU)} & 9.72 & 9.72 & 0.01 & 0 & 106 & 10 \\
{Hybrid DNC, H-A2C} & 9.59 & 9.75 & 0.27 & 0 & 107 & 10 \\
{Hybrid Cacla, DDPG} & 4.39 & 4.71 & 3.77 & -51 & 0 & 10 \\
\midrule
\midrule
\multicolumn{7}{c}{Recommender, $343^1$, hybrid} \\
\midrule
{DGRL (SDN, DBU)} & 1579.64 & 1577.34 & 23.34 & 11 & 17 & 10 \\
{Hybrid DNC, H-A2C} & 1362.25 & 1412.97 & 144.42 & 0 & 5 & 10 \\
{Hybrid Cacla, DDPG} & 1335.42 & 1342.34 & 34.54 & -4 & 0 & 10 \\
{HyAR, A2C} & 607.89 & 601.09 & 24.74 & -57 & -55 & 10 \\
\midrule
\midrule
\multicolumn{7}{c}{Recommender, $343^{10}$, hybrid} \\
\midrule
{DGRL (SDN, DBU)} & 2262.56 & 2258.83 & 60.81 & 129 & 29 & 10 \\
{Hybrid DNC, H-A2C} & 1011.78 & 982.57 & 112.17 & 0 & -43 & 10 \\
{Hybrid Cacla, DDPG} & 1740.56 & 1748.69 & 43.23 & 77 & 0 & 10 \\
\midrule
\bottomrule
\end{tabularx}
\label{tab:hybr_res}
\end{table}


\newpage
\section*{NeurIPS Paper Checklist}

\begin{enumerate}

\item {\bf Claims}
    \item[] Question: Do the main claims made in the abstract and introduction accurately reflect the paper's contributions and scope?
    \item[] Answer: \answerYes{} 
    \item[] Justification: All claims made in the Abstract, Introduction, and Conclusion are grounded either in theoretical or in empirical results presented in the main body of the paper or in the appendices. Specifically, theoretical claims are backed by Section~\ref{sec:analysis} and Appendix~\ref{app:theory}, while empirical results are backed by Section~\ref{sec:exp_setup} and Appendix~\ref{app:results}.

\item {\bf Limitations}
    \item[] Question: Does the paper discuss the limitations of the work performed by the authors?
    \item[] Answer: \answerYes{} 
    \item[] Justification: We discuss the limitations of our proposed algorithm in Section~\ref{sec:sdn} and the assumptions and limitations of our theoretical analysis in Section~\ref{sec:analysis}.

\newpage

\item {\bf Theory assumptions and proofs}
    \item[] Question: For each theoretical result, does the paper provide the full set of assumptions and a complete (and correct) proof?
    \item[] Answer: \answerYes{} 
    \item[] Justification: We provide proofs for all theoretical results of Section~\ref{sec:analysis} either directly there, or in Appendix~\ref{app:theory} or~\ref{app:hybrid}.

    \item {\bf Experimental result reproducibility}
    \item[] Question: Does the paper fully disclose all the information needed to reproduce the main experimental results of the paper to the extent that it affects the main claims and/or conclusions of the paper (regardless of whether the code and data are provided or not)?
    \item[] Answer: \answerYes{} 
    \item[] Justification: We provide a detailed description of our proposed algorithm in Section~\ref{sec:methodology}. We provide a detailed description of implementation details in Appendix~\ref{app:implementation}, and of all experimental details, incl. environments, baselines, and hyperparameters, in Appendix~\ref{app:experiments}. These appendices cover the information required to reproduce our experiments. In addition, we provide our source code in the repository: \url{https://github.com/HeikoFHH/Distance-Guided-RL}.

\item {\bf Open access to data and code}
    \item[] Question: Does the paper provide open access to the data and code, with sufficient instructions to faithfully reproduce the main experimental results, as described in supplemental material?
    \item[] Answer: \answerYes{} 
    \item[] Justification: We publish the code to a Python repository on Github. All other Python packages we used are publicly available. Our code contains instructions to reproduce our experimental results, the data used to generate the results of the paper, as well as comments to understand key concepts. The repository can be found under: \url{https://github.com/HeikoFHH/Distance-Guided-RL}.

\item {\bf Experimental setting/details}
    \item[] Question: Does the paper specify all the training and test details (e.g., data splits, hyperparameters, how they were chosen, type of optimizer) necessary to understand the results?
    \item[] Answer: \answerYes{} 
    \item[] Justification: We provide a concise overview over our experiments in Section~\ref{sec:exp_setup}, containing the most important experiment specifications. We provide a comprehensive and detailed description of all relevant details required to understand our experiments in Appendix~\ref{app:experiments}. In addition, our code contains the specifications of all environments.

\item {\bf Experiment statistical significance}
    \item[] Question: Does the paper report error bars suitably and correctly defined or other appropriate information about the statistical significance of the experiments?
    \item[] Answer: \answerYes{} 
    \item[] Justification: We present our main results in Section~\ref{sec:exp_setup} as lineplots, displaying median performances and interquartile ranges to show the distribution of our results. Early investigations revealed that there are very few outliers beyond the IQRs. For all experiments, we conduct ten training runs with random model initializations. In addition to that, we display performance means, medians, and standard deviations in the tables in Appendix~\ref{app:results}. Assumptions on the distribution of e.g., errors are not necessary in our experiments. Overall, these measures show the statistical significance of our results.

\item {\bf Experiments compute resources}
    \item[] Question: For each experiment, does the paper provide sufficient information on the computer resources (type of compute workers, memory, time of execution) needed to reproduce the experiments?
    \item[] Answer: \answerYes{} 
    \item[] Justification: We outline the compute resources used to generate the results in this paper in Appendix~\ref{app:experiments}. As stated there, these are the same resources we used for all algorithmic development.
    
\item {\bf Code of ethics}
    \item[] Question: Does the research conducted in the paper conform, in every respect, with the NeurIPS Code of Ethics \url{https://neurips.cc/public/EthicsGuidelines}?
    \item[] Answer: \answerYes{} 
    \item[] Justification: This paper does not violate any part of the NeurIPS Code of Ethics. Neither the algorithm we propose nor the experiments we conduct have caused harm to anyone, nor do we expect them to cause harm to anyone.

\item {\bf Broader impacts}
    \item[] Question: Does the paper discuss both potential positive societal impacts and negative societal impacts of the work performed?
    \item[] Answer: \answerYes{} 
    \item[] Justification: Since this paper is a pure methodological contribution, we do not expect it to have any direct societal impact. We briefly discuss the positive impact our algorithm may have on industrial applications in Section~\ref{sec:introduction} and on research in Section~\ref{sec:conclusion}.
    
\item {\bf Safeguards}
    \item[] Question: Does the paper describe safeguards that have been put in place for responsible release of data or models that have a high risk for misuse (e.g., pre-trained language models, image generators, or scraped datasets)?
    \item[] Answer: \answerNA{} 
    \item[] Justification: We do not see potential for misuse of the algorithm proposed in this paper.

\item {\bf Licenses for existing assets}
    \item[] Question: Are the creators or original owners of assets (e.g., code, data, models), used in the paper, properly credited and are the license and terms of use explicitly mentioned and properly respected?
    \item[] Answer: \answerYes{} 
    \item[] Justification: All external assets used in the paper -- specifically environments, packages and benchmarks -- are properly credited through an extensive list of references. No additional code, data, or models requiring license or terms of use considerations were used.

\item {\bf New assets}
    \item[] Question: Are new assets introduced in the paper well documented and is the documentation provided alongside the assets?
    \item[] Answer: \answerYes{} 
    \item[] Justification: We upload the code to a GitHub repository, containing all documentation and explanations: \url{https://github.com/HeikoFHH/Distance-Guided-RL}.

\item {\bf Crowdsourcing and research with human subjects}
    \item[] Question: For crowdsourcing experiments and research with human subjects, does the paper include the full text of instructions given to participants and screenshots, if applicable, as well as details about compensation (if any)? 
    \item[] Answer: \answerNA{} 
    \item[] Justification: This paper does not conduct crowdsourcing experiments, nor does it experiment with human subjects.

\item {\bf Institutional review board (IRB) approvals or equivalent for research with human subjects}
    \item[] Question: Does the paper describe potential risks incurred by study participants, whether such risks were disclosed to the subjects, and whether Institutional Review Board (IRB) approvals (or an equivalent approval/review based on the requirements of your country or institution) were obtained?
    \item[] Answer: \answerNA{} 
    \item[] Justification: This paper does not conduct crowdsourcing experiments, nor does it experiment with human subjects.

\item {\bf Declaration of LLM usage}
    \item[] Question: Does the paper describe the usage of LLMs if it is an important, original, or non-standard component of the core methods in this research? Note that if the LLM is used only for writing, editing, or formatting purposes and does \emph{not} impact the core methodology, scientific rigor, or originality of the research, declaration is not required.
    \item[] Answer: \answerNA{} 
    \item[] Justification: LLMs do not constitute an important, original, or non-standard component of the core methods in this paper. We used LLMs only for minor debugging, editing, and formatting purposes. LLMs did not contribute to the core research of this paper.

\end{enumerate}

\end{document}